%% file: tmlr.tex
\documentclass[10pt, preprint]{article} 
\usepackage{tmlr}
\usepackage{algorithm}
\usepackage{algorithmic}

\input{math_commands.tex}

\usepackage{hyperref}
\usepackage{url}

\title{Fixed-Budget Best-Arm Identification in
\\Sparse Linear Bandits}

\author{\name Recep Can Yavas \email recep.yavas@cnrsatcreate.sg \\
      \addr 
      CNRS at CREATE, Singapore
      \AND
      \name Vincent Y. F. Tan \email vtan@nus.edu.sg \\
      \addr Department of Mathematics, \\  Department of Electrical and Computer Engineering, \\ National University of Singapore
      }

\def\month{MM}  
\def\year{YYYY} 
\def\openreview{\url{https://openreview.net/forum?id=XXXX}} 

\begin{document}

\maketitle

\begin{abstract}
We study the best-arm identification problem in sparse linear bandits under the fixed-budget setting. In sparse linear bandits, the unknown feature vector $\ts$ may be of large dimension $d$, but only a few, say $s \ll d$ of these features have non-zero values. We design a two-phase algorithm, Lasso and Optimal-Design- (Lasso-OD) based linear best-arm identification. The first phase of Lasso-OD leverages the sparsity of the feature vector by applying the thresholded Lasso introduced by Zhou (2009), which  estimates the support of $\ts$ correctly  with high probability using rewards from the selected arms and a judicious choice of the design matrix. The second phase of Lasso-OD applies the OD-LinBAI algorithm by Yang and Tan (2022) on that estimated support.  We derive a non-asymptotic upper bound on the error probability of Lasso-OD by carefully choosing hyperparameters (such as Lasso's regularization parameter) and balancing the error probabilities of both phases. For fixed sparsity $s$ and budget $T$, the exponent in the error probability of Lasso-OD depends on $s$ but not on the dimension $d$,  yielding a significant performance improvement for sparse and high-dimensional linear bandits. Furthermore, we show that Lasso-OD is almost minimax optimal in the exponent. Finally, we provide  numerical examples to demonstrate the significant performance improvement over the existing algorithms for non-sparse linear bandits such as OD-LinBAI, BayesGap, Peace, LinearExploration, and GSE.
\end{abstract}

\section{Introduction}
The stochastic multi-armed bandit (MAB) is a model that provides a mathematical formulation to study the sequential design of experiments and exploration-exploitation trade-off, where a learner pulls an arm out of a total $K$ and receives a reward drawn from a fixed and unknown distribution according to the chosen arm. This model has several applications including online advertising, recommendation
systems, and drug tests. While in the standard reward model, the arms are uncorrelated with each other, stochastic linear bandits introduced in \citet{auer2002} generalize the standard model by associating each arm with a $d$-dimensional feature vector and the reward is equal to the inner product between the feature vector and an unknown global parameter. Therefore, the arms are correlated in linear bandits, meaning that pulling an arm gives information about the rewards of some other arms.

Most prior work including \citet{auer2002, thompson1933, robbins1952some, bubeck2012regret, dani2008stochastic} on MABs focuses on \emph{regret minimization}, where the goal is to maximize the cumulative reward after $T$ arm pulls by optimizing the trade-off between exploration and exploitation. Recently, the \emph{pure exploration} setting has drawn attention from researchers. One example of pure exploration is the \emph{best-arm identification} (BAI) problem, where the goal is to identify the arm with the largest mean reward. The BAI problem is studied in two settings: (1) the fixed-budget setting considers a budget $T \in \mathbb{N}$ and aims to minimize the probability of failing to identify the best in at most $T$ arm pulls; (2) the fixed-confidence setting considers a confidence level $\delta \in (0, 1)$ and aims to minimize the average number of arm pulls while identifying the best arm with probability at least $1-\delta$. 

For the standard reward model with uncorrelated arms, the works in \citet{evans, karnin, kauffmann} and \citet{carpentier, audibert} consider the BAI problem in the fixed-confidence and fixed-budget settings, respectively. For the linear model, the works in \citet{soare, xu, fiez, tao, jedra, peleg} develop several algorithms under the fixed-confidence setting. For the linear model under the fixed-budget setting, \citet{hoffman} develop the first algorithm, BayesGap, which is a gap-based exploration algorithm using a Bayesian approach. \citet{katzsamuels} develop the Peace algorithm that has equally-sized rounds, where the arm-pulling strategy within each round is based on the Gaussian width of the underlying arm set. \citet{alieva} develop LinearExploration that exploits the linear structure of the model and is robust to unknown
levels of observation noise and misspecification in the linear model. \citet{yang} develop the Optimal-Design-Based Linear Best Arm Identification (OD-LinBAI) algorithm, which also employs almost equally-sized rounds, but the arm-pulling strategy within each round is based on the G-optimal design. In the first round, OD-LinBAI aggressively eliminates all empirically suboptimal arms except the top $\frac{d}{2}$ arms; in the subsequent rounds, half of the remaining arms are eliminated in each round until a single arm remains. \citet{azizi} develop the Generalized Successive Elimination (GSE) algorithm that has similar principles as OD-LinBAI with the difference that GSE eliminates the half of the remaining arms in all rounds. Among these algorithms, only OD-LinBAI is shown to be asymptotically minimax optimal.  

In many practical applications of MABs, there are a large number of features available to the learner, but only a few of these features significantly affect the value of the reward of an arm. Sparse linear bandits are a mathematical abstraction of this phenomenon by assuming that the $d$-dimensional unknown parameter $\ts$ in the linear model has only $s$ nonzero values, i.e., $\norm{\ts}_0 = s$, where $s$ is usually much smaller than $d$. The performance in the MAB problems (e.g., cumulative regret, probability of identification error) usually deteriorates as the ambient dimension $d$ increases. Therefore, the goal in the sparse setting is to design an algorithm whose performance is a function of $s$ but not $d$. Some works that study the regret minimization problem for sparse linear bandits include \citet{abbasi, kimpaik2019, hao, oh2021, ariu, li2022, jang2022, wang2023, chak2023}. The OFUL algorithm of \citet{abbasi} keeps track of a high probability confidence set for $\ts$ and pulls an arm that maximizes the reward with respect to the arm vectors and the confidence set for $\ts$. The DR Lasso algorithm of \citet{kimpaik2019} combines Lasso with a doubly-robust technique used in the missing data literature. The ESTC algorithm of \citet{hao} uses Lasso to estimate $\ts$ at the end of the first phase and then in the second phase commits to the best arm with respect to the Lasso estimate. The SA Lasso Bandit algorithm of \citet{oh2021} estimates $\ts$ at each time using Lasso and pulls the best arm with respect to the Lasso estimate. TH Lasso Bandit algorithm of \citet{ariu} estimates the support of $\ts$ using Lasso and a thresholding procedure at each time and pulls the best arm with respect to the ordinary least squares estimation restricted to the estimated support in the first phase. \citet{li2022} generalize the ESTC algorithm of \citet{hao} to general bandit problems with low-dimensional structures such as low-rank matrix bandits. The PopArt algorithm of \cite{jang2022} takes the population covariance of arms as input and uses a thresholding step to estimate $\ts$ in the first phase; in the second phase, it commits to the best arm with respect to the estimate of $\ts$ in the first phase. The LRP-Bandit algorithm of
\citet{wang2023} combines the thresholded Lasso with random projection where random projection is used to mitigate the
negative influence of model misspecification due to the Lasso phase; their algorithm is also computationally efficient since Lasso is computed only at times with exponentially increasing gaps.
Finally, \citet{chak2023} develop a Thompson Sampling algorithm for sparse linear contextual bandits. 

In this paper, we study the BAI problem in sparse linear bandits under the fixed-budget setting. To the best of our knowledge, this paper presents the first result on the BAI problem in linear bandits with sparse structure, and we show that our bound on the error probability is almost minimax optimal in the exponent.

\paragraph{Contributions} Our main contributions are summarized as follows.
\begin{enumerate}
    \item We design an algorithm, \emph{Lasso and Optimal-Design-  (Lasso-OD) based Linear Best Arm Identification}. This algorithm has two phases. In the first phase, we pull arms to estimate a support set $\hat{\mathcal{S}}$ that captures the support of the unknown parameter $\ts$ with high probability and has size as small as possible. This goal is accomplished by the thresholded Lasso (TL) introduced by \citet{zhou2009}. TL obtains an initial estimation $\tinit$ for the parameter $\ts$ from Lasso \citep{tibshirani} and passes it through an absolute value threshold to obtain $\tthres$. The support of $\tthres$ is the output of the first phase. In the second phase, we apply OD-LinBAI from \citet{yang}. Lasso-OD has~3 hyperparameters: (i) $T_1 < T$, the budget allocated for the first phase; (ii) $\lambda_{\mr{init}} > 0$, the parameter in the initial Lasso problem; and (iii) $\lambda_{\mr{thres}} > 0$, the threshold value in TL. The choice of the design matrix (i.e., number of times each arm is pulled) in the first phase is crucial in attaining a good performance. Inspired by \citet{hao}, we design it by maximizing the smallest eigenvalue of the Gram matrix associated with the design matrix; this is known as the {\em E-optimal design}~\cite[Sec.~7.5.2]{Boyd04}. 
    This particular choice minimizes an upper bound on a probability term related to the performance of TL. 

    \item We derive a non-asymptotic upper bound on the error probability of Lasso-OD as a function of the total budget $T$, the number of arms $K$, the ambient dimension $d$, the sparsity $s$, and the arm vectors $\va(k)$, $k = 1, \dots, K$, the first few suboptimality gaps, and the hyperparameters $T_1, \lambda_{\mr{init}}$, and $\lambda_{\mr{thres}}$. As a corollary to this bound, with the knowledge of $s$, we carefully choose the hyperparameters so that firstly,  with high probability, phase 1 selects all variables in $\ts$ and at most $s^2$ additional variables and secondly, the probability terms due to phases 1 and 2 are approximately ``balanced''.  This particular choice achieves the error probability $\exp \big\{ - \Omega \big(\frac{T}{ (\log_2 s) H_{2, \mr{lin}}(s + s^2)} \big)\big\}$ for fixed $s$, $T \to \infty$, $K$ and $d$ not growing exponentially with $T$ (see Corollary~\ref{cor:main}). Here, $H_{2, \mr{lin}}(s + s^2)$ is a hardness parameter that depends only on the first $s + s^2 - 1$ suboptimality gaps. Note that the exponent is independent of dimension $d$, implying that increase in $d$ does not significantly increase the error probability. For OD-LinBAI, this exponent is given by $\exp \big\{ - \Omega \big(\frac{T}{ (\log_2   d) H_{2, \mr{lin}}(d)} \big)\big\}$; therefore, Lasso-OD improves the error probability exponent by a factor of $\Omega(\frac{\log_2 d}{\log_2 s})$ for $d \geq s + s^2$. 
    \item We empirically compare the identification error of Lasso-OD with that of other existing algorithms in the literature on several synthetic datasets, including one that is a sparsity-based version of  examples used in other papers~\citep{jedra}. The empirical results support our theoretical result that claims that the scaling of the error probability of Lasso-OD is characterized by the sparsity $s$ while the performances of other algorithms significantly depend on $d$.
\end{enumerate}

\section{Problem Formulation}

We consider a standard linear bandit with $K$ arms with a $d$-dimensional unknown global parameter~$\ts$. Let the arm set be $[K] \triangleq \{1, \dots, K\}$, where each arm $k \in [K]$ is associated with a known arm vector $\va(k) \in \mathbb{R}^d$. A set of $K$ arms, $\{\va(1), \dots, \va(k)\}$, together with $\ts$ define a linear bandit instance $\eta$. At each time $t$, the agent chooses an arm $A_t \in [K]$ and observes a noisy reward
\begin{align} 
    \ry_t = \langle \ts, \va(A_t) \rangle + \epsilon_t,
\end{align}
where $\epsilon_1, \epsilon_2, \dots$ are independent 1-subgaussian noise variables. For the arm selection, the agent uses an online algorithm, that is, the arm pull $A_t \in [K]$ may depend only on the previous $t-1$ arm pulls $A_1, \dots, A_{t-1}$ and their corresponding rewards $\ry_1, \dots, \ry_{t-1}$. Denote the mean rewards of the arm vectors by
\begin{align}
    \mu_k \triangleq \langle \ts, \va(k) \rangle, \quad \forall \, k \in [K].
\end{align}
Without loss of generality, we assume that $\mu_1 > \mu_2 \geq \mu_3 \geq \dots \geq \mu_K$, i.e., arm 1 is the unique best arm. We denote the mean gaps by
$\Delta_k \triangleq \mu_1 - \mu_k$ for $2 \leq k \leq K$.

Under the fixed-budget setting of BAI, the agent is given a fixed time $T$, and makes an estimate $\hat{I}$ for the best arm with no more than $T$ arm pulls. The goal is to design an online algorithm with the identification error probability, $\mathbb{P}[\hat{I} \neq 1]$, as small as possible. 

\textbf{Notation: }
For any integer $n$, we denote $[n] \triangleq \{1, \dots, n\}$. Let $\vx = (x_1, \dots, x_d)$ be a $d$-dimensional vector and $\mathcal{S} \subseteq [d]$, we denote $\vx_{\mc{S}} \triangleq (x_s \colon s \in \mathcal{S}) \in \mathbb{R}^{|\mathcal{S}|}$. We denote $\norm{\vx}_{\mA} \triangleq \sqrt{\vx^\top \mA \vx}$. The minimum eigenvalue of a symmetric $\mA$ is denoted by $\sigma_{\min}(\mA)$. We denote the set of distributions on the set $\mathcal{A}$ as $\mathcal{P}(\mathcal{A})$. 
Let $A_1, \dots, A_t \in [K]$ be a sequence of arm pulls. The matrix $\mX \in \mathbb{R}^{t \times d}$ whose $j$-th row is $\va(A_j)^\top$ is called the \emph{design matrix}. Let $\nu \in \mathcal{P}([K])$ be the vector of fractions of arm pulls associated with this strategy, i.e., $\nu_k = \frac{1}{t}\sum_{j = 1}^t 1\{A_j = k\} $ for $k \in [K]$. The \emph{Gram matrix} associated with this strategy is denoted by $\mM(\nu) = \frac{1}{t} \mX^\top \mX = \sum_{k \in [K]} \nu_k \va(k) \va(k)^\top \in \mathbb{R}^{d \times d}$. When we use asymptotic notation such as $O(\cdot)$ and $\Omega(\cdot)$, somewhat unconventionally, we are referring to {\em nonnegative} sequences, e.g., $a_n \in O(b_n)$ if and only if $\limsup_{n\to\infty} \frac{a_n}{b_n }<\infty$ and $\{a_n\}_{n\ge1}$ is a nonnegative sequence.

\textbf{Model assumptions: } Denote the support of $\ts$ by $S(\ts) \triangleq \{j \in [d] \colon \theta^*_j \neq 0\}$. We assume that the unknown parameter $\ts$ and the arm vectors $\{\va(k)\}_{k \in [K]}$ are  of length $d$ 
but $\ts$ is sparse, i.e., the number of non-zero coefficients in $\ts$ satisfies $\norm{\ts}_0 \triangleq |S(\ts)| = s < d$. We assume that $S(\ts)$ is unknown, but $s$ and $\theta_{\min} \triangleq \min_{j \in S(\ts)} |\theta^*_j|$ are known. We further assume that $|\mu_k| \leq 1$ for all arms $k \in [K]$. 

\section{Our Algorithm: Lasso-OD}
We now present our algorithm, \emph{Lasso and Optimal-Design- (Lasso-OD) based linear best-arm identification} which has two phases. In phase 1, we pull a judiciously chosen set of arms to learn   the support of the unknown parameter $\ts$. Specifically, we design  phase 1 so that it outputs a subset of variables $\hat{\mathcal{S}} \subseteq [d]$ whose  support $\hat{\mathcal{S}}$ captures the true variables, $S(\ts)$, with high probability, and its cardinality $|\hat{\mathcal{S}}|$ is small. To do this, we use the thresholded Lasso introduced by \citet{zhou2009}. Once $\hat{\mathcal{S}}$ is obtained, we eliminate all variables in the arm vectors except the ones in $\hat{\mathcal{S}}$. Note that given that $\hat{\mathcal{S}} \supseteq S(\ts)$, this variable elimination would have no effect on the mean values $\mu_1, \dots, \mu_K$ since by assumption, we only eliminate some variables $j \in [d]$ with $\theta^*_j = 0$. Therefore, the best arm is also preserved after   variable elimination. Building upon this principle, 
in phase 2, we project the arms on the estimated support $\hat{\mc{S}}$ and pull arms according to the OD-LinBAI algorithm by \citet{yang}, which is designed for linear bandits without the sparsity assumption. 

\subsection{Motivation for Lasso-OD Algorithm}
OD-LinBAI used in phase 2 is a minimax optimal algorithm up to a multiplicative factor in the exponent in the sense that it achieves an asymptotic error probability $\exp \big\{ - \Omega \big(\frac{T}{ (\log_2  d) H_{2, \mr{lin}}(d)} \big)\big\}$, and for every algorithm, there exists a bandit instance $\eta$ whose asymptotic error probability is lower bounded by $\exp \big\{ - O \big(\frac{T}{ (\log_2  d) H_{2, \mr{lin}}(d)} \big)\big\}$. The hardness parameter 
\begin{align}
    H_{2, \mathrm{lin}}(d) \triangleq \max_{2 \leq i \leq d} \frac{i}{\Delta_i^2}
\end{align}
determines how difficult it is to identify the best arm for a given bandit instance $\eta$ \citep{yang}. For sparse linear bandits, if an oracle knew the support of the unknown parameter $\ts$, then the lower bound in \citet[Th.~3]{yang} would be improved to $\exp \big\{ - O \big(\frac{T}{ (\log_2  s) H_{2, \mr{lin}}(s)} \big)\big\}$. The purpose of TL in phase~1 is to provide an estimate for the support of $\ts$ with high accuracy while also pulling arms few enough that the resulting error probability is a function of $s$ rather than $d$ as in the oracle lower bound. 
Below, we provide the details on two phases of Lasso-OD.

\subsection{Phase 1 (TL)}
Consider a linear model $\rvy = \mX \ts + \bm{\epsilon}$, where $\mX \in \mathbb{R}^{T_1 \times d}$ is a fixed design matrix, $\ts \in \mathbb{R}^d$ is a fixed unknown feature vector, $\rvy \in \mathbb{R}^{T_1}$ is the response vector, and $\bm{\epsilon} \in \mathbb{R}^{T_1}$ is a noise vector whose entries are independent and 1-subgaussian. \citet{tibshirani} introduces the Lasso optimization problem to identify a sparse solution to the least squares estimation problem
\begin{align}
    \that_{\mr{init}} = \argmin_{\vtheta \in \mathbb{R}^d} \frac{1}{T_1} \norm{\rvy - \mX \vtheta}_2^2 + \lambda_{\mr{init}} \norm{\vtheta}_1, \label{eq:lassoinit}
\end{align}
where $\lambda_{\mr{init}} > 0$ is a suitably chosen regularization parameter. The Lasso \eqref{eq:lassoinit} is a convex program and can be solved efficiently, e.g., using Alternating Direction Method of Multipliers (ADMM) algorithm \citep{boyd2011}.

For the task of variable selection, i.e., recovering the support of the unknown parameter $\ts$ without missing any of its non-zero variables, we want to obtain an estimate $\that$ that satisfies $S(\that) \supseteq S(\ts)$ while ensuring that $| S(\that) \setminus S(\ts)|$ is as small as possible. \citet{zhou2009} introduces the following thresholding procedure that has this property
\begin{align}
(\tthres)_j &= (\tinit)_j \, 1\{ |(\tinit)_j| \geq \lambda_{\mathrm{thres}} \}, \quad \forall \, j \in [d], \label{eq:threshlas}
\end{align}
where the initial estimate $\that_{\mr{init}}$ is given in \eqref{eq:lassoinit}, and $\lambda_{\mathrm{thres}} > 0$ is the threshold. The set of selected variables by TL is $S(\tthres)$. A variation of TL is used by \citet{ariu} to derive refined regret guarantees in sparse stochastic contextual linear bandits. Their main idea is to find the support estimate $S(\tthres^{(t)})$ at each time instance $t$ using TL and then to compute the ordinary least squares (OLS) estimation restricted on the variables in  $S(\tthres^{(t)})$. \citet{ariu}  tune the free parameters $\lambda_{\mathrm{init}}^{(t)}$ and $\lambda_{\mathrm{thres}}^{(t)}$ in a way that with high probability, $S(\tthres^{(t)}) \supseteq S(\ts)$ and $S(\tthres^{(t)})$ is small enough, which is $s + O(\sqrt{s})$ in their case. Note that on the event $\{S(\tthres^{(t)}) \supseteq S(\ts)\}$, the OLS solution restricted on the subset $S(\tthres^{(t)})$ is equal to that for the unrestricted case where all $d$ variables are used.
Our approach is similar to that in \citet{ariu} in using TL to reduce the effective dimension of the problem.

Let $T_1 < T$ be the budget allocated to the variable selection procedure described above. 

\paragraph{Design matrix optimization} 
First, we need to specify the number of pulls for each arm during phase 1, which corresponds to determining the design matrix $\mX \in \mathbb{R}^{T_1 \times d}$ in the Lasso problem \eqref{eq:lassoinit}. To do this, we solve the optimization problem, known as the {\em E-optimal design}~\cite[Sec.~7.5.2]{Boyd04}, given by
\begin{align}
    \tilde{\nu}^\star = \argmax_{\nu \in \mathcal{P}([K])} \sigma_{\min}\left(\sum_{i = 1}^K \nu_i \va(i) \va(i)^\top \right). \label{eq:vtilde}
\end{align}
Since the function $\mA \mapsto \sigma_{\min}(\mA)$ is concave and $\nu \mapsto \sum_{i = 1}^K \nu_i \va(i) \va(i)^\top$ is linear, \eqref{eq:vtilde} is a convex optimization problem, and can be solved efficiently, for example, using the CVX toolbox \citep{boyd2011}. 

The design matrix determined by the allocation in \eqref{eq:vtilde} minimizes an upper bound on a probability term related to phase~1; hence, it approximately optimizes the penalty term due to incorrectly estimating the variables of $\ts$. More discussion on this choice of the design matrix appears in Appendix~\ref{app:design}. The optimization problem~\eqref{eq:vtilde} also appears in \citet{hao} on their regret analysis in sparse linear bandits. The allocation $\tilde{\nu}^\star$ can lead to fractional number of pulls $T_1 \tilde{\nu}^\star_i$ for some arm $i \in [K]$. To guarantee integer number of pulls for all arms, we apply a rounding procedure given in \citet[Ch.~12]{pukel}, the $\mathrm{ROUND}$ function in Appendix~\ref{app:algorithms}, which is also employed in the fixed-confidence BAI algorithm~in~\citet{fiez}. 

\paragraph{Support estimation}
We compute the number of pulls for each arm using \eqref{eq:vtilde} and $\mathrm{ROUND}$, and then estimate the support from \eqref{eq:lassoinit} and~\eqref{eq:threshlas}. Algorithm~\ref{alg:ESP} below delineates the pseudo-code of this procedure.
\begin{algorithm}[!htbp] 
   \caption{Thresholded Lasso (TL)}
   \label{alg:ESP}
\begin{algorithmic}[1]
\INPUT{Time budget $T_1$, Lasso parameters $\lambda_{\mathrm{init}}$ and $\lambda_{\mathrm{thres}}$, and arm vectors $\va(1), \dots, \va(K)$.}
\STATE Compute the arm pull fractions $\tilde{\nu}^*$ from \eqref{eq:vtilde}.
\STATE Update $\tilde{\nu}^* \gets \textrm{ROUND}(\tilde{\nu}^*, T_1)$ to ensure integer number of arm pulls.
\STATE Pull each arm $i \in [K]$ exactly $T_1 \tilde{\nu}^*_i$ times. Denote the vector of rewards by $\rvy \in \mathbb{R}^{T_1}$.
\STATE Form the design matrix $\mX \in \mathbb{R}^{T_1 \times d}$ so that it has $T_1 \tilde{\nu}^*_i$ rows equal to $\va(i)^\top$ for $i \in [K]$. Compute $\tthres$ from \eqref{eq:lassoinit} and~\eqref{eq:threshlas}.
\OUTPUT the support $\hat{\mathcal{S}} = S(\tthres)$.
\end{algorithmic}
\end{algorithm}

\subsection{Phase 2 (OD-LinBAI)}

In this section, we review the OD-LinBAI algorithm by \citet{yang}. 
OD-LinBAI divides the budget $T$ into $\lceil \log_2 d \rceil$ phases, where each phase has roughly the same length.  

At the start of round $r$, OD-LinBAI applies a dimensionality reduction step to maintain that the set of modified arms spans the space of its reduced dimension. 
The arm allocation during each round is determined by the {\em G-optimal design} \citep{kiefer}, which takes a  set of arm vectors $\{\va(1), \dots, \va(K)\} \subseteq \mathbb{R}^d$ and solves the optimization problem
\begin{align}
    \pi^* = \argmin_{\pi \in \mathcal{P}([K])} \max_{i \in [K]} \norm{ \va(i) }_{\mM(\pi)^{-1}}^2, \label{eq:goptimal}
\end{align}
where $\mM(\pi) \triangleq \sum_{i = 1}^K \pi_i \va(i) \va(i)^\top$ is the Gram matrix associated with the allocation $\pi$.  
At the start of each round, we solve \eqref{eq:goptimal} for the set of active arms and then apply the $\mr{ROUND}$ function in Appendix~\ref{app:algorithms} to the resulting allocation to ensure integer number of pulls. The latter step 
replaces the procedure in Line~17  \citet[Algorithm~1]{yang}. This slight modification may improve the performance of the algorithm especially if the budget $T$ is small. At the end of round 1, we eliminate all arms except the top $\lceil \frac{d}{2} \rceil$ with respect to the OLS estimator; in the rest, we halve the remaining arms at the end each round. At the end of last round, only one arm remains and that arm is declared to be the best one.  The pseudo-code of OD-LinBAI can be found in \citet{yang} and a slight modification of it which leads to the improved error probability bound in  Theorem~\ref{thm:ODLinBAI} can be found in Appendix~\ref{app:algorithms}.

\subsection{Lasso-OD Algorithm}
The pseudo-code of Lasso-OD described above is given in Algorithm~\ref{alg:lasso-OD}. Notice that since the two phases of Lasso-OD operate independently, one can replace either or both of TL and OD-LinBAI with their alternatives, e.g., the PopArt algorithm \citep{jang2022} and the adaptive Lasso \cite[Ch.~2.8]{vandegeerbook} for TL and any of the algorithms in \citet{alieva, katzsamuels, azizi, hoffman} for OD-LinBAI. We discuss some of the variants of our algorithm in Appendix~\ref{app:experiments}. 
\begin{algorithm}[!htbp] 
   \caption{Lasso and Optimal-Design Based Linear Best Arm Identification (Lasso-OD)}
   \label{alg:lasso-OD}
\begin{algorithmic}[1]
\INPUT Time budgets $T_1$ and $T_2$ so that $T = T_1 + T_2$, Lasso parameters $\lambda_{\mr{init}}$ and $\lambda_{\mr{thres}}$, and arm vectors $\va(1), \dots, \va(K) \in \mathbb{R}^d$. 
\STATE Run TL (Algorithm~\ref{alg:ESP}) with $T_1, \lambda_{\mathrm{init}}$, and $\lambda_{\mathrm{thres}}$ and get the output $\hat{\mathcal{S}} \subseteq [d]$.
\STATE Project the arm vectors on the subset $\hat{\mathcal{S}}$ by setting $\va'(i) = {(\va(i))}_{\hat{\mathcal{S}}}$ for $i \in [K]$.
\STATE Run OD-LinBAI from \citet{yang}  with budget $T_2$ and arm vectors $\{\va'(1), \dots, \va'(K)\} \subseteq \mathbb{R}^{|\hat{\mathcal{S}}|}$ with Line 17 of Algorithm 1 in \citet{yang} replaced by $\mathrm{ROUND}$.
\OUTPUT the only remaining arm $\hat{I}$ as the output of OD-LinBAI.
\end{algorithmic}
\end{algorithm}

\section{Main Results}
This section presents three non-asymptotic upper bounds on the performances of TL, OD-LinBAI, and Lasso-OD algorithms.
\subsection{Thresholded Lasso}
Recall the linear model $\rvy = \mX \ts + \bm{\epsilon}$, where $\mX \in \mathbb{R}^{T_1 \times d}$ is a fixed design matrix, $\ts \in \mathbb{R}^d$ is a fixed unknown feature vector, $\rvy \in \mathbb{R}^{T_1}$ is the response vector, and $\bm{\epsilon} \in \mathbb{R}^{T_1}$ is a noise vector whose entries are independent and 1-subgaussian. For any set $\mathcal{S} \subseteq [d]$, define the set of vectors
\begin{align}
    \mathbb{C}(\mc{S}) \triangleq \{\vtheta \in \mathbb{R}^d \colon \norm{\vtheta_{\mc{S}^\mathrm{c}}}_1 \leq 3 \norm{\vtheta_{\mc{S}}}_1 \}.
\end{align}
\citet{vandegeer} introduce the following {\em  compatibility condition} that allows one to control the $\ell_1$-norm error for the sparse estimation of the unknown parameter $\ts$ where the components of the design matrix $\mX$ are not highly correlated. 
For the rest of the section, let $\mM = \frac{1}{T_1} \mX^\top \mX$ denote the Gram matrix associated with $\mX$.
\begin{definition}[Compatibility condition] \label{def:comp}
Given a fixed design matrix $\mX \in \mathbb{R}^{T_1 \times d}$ (whose Gram matrix is $\mM$) and a subset $\mathcal{S} \subseteq [d]$, the compatibility constant $\phi^2(\mM, \mathcal{S})$ is defined as
\begin{align}
    \phi^2(\mM, \mathcal{S}) \triangleq \min_{\vtheta \in \mathbb{R}^d \colon \norm{\vtheta_{\mc{S}}}_1 \neq 0} \left\{ \frac{|\mathcal{S}| \norm{\vtheta}_{\mM}^2}{ \norm{\theta_{\mc{S}}}_1^2} \colon \vtheta \in \mathbb{C}(\mc{S}) \right \}. 
\end{align}
With some abuse of notation, we also define
\begin{align}
    \phi^2(\mM, s) \triangleq \min_{\mathcal{S} \subseteq [d] \colon |\mathcal{S}| = s} \phi^2(\mM, \mathcal{S}).
\end{align}
\end{definition}
The following result controls the $\ell_1$-norm error of the initial Lasso estimator in~\eqref{eq:lassoinit}.
\begin{lemma}[\citet{ariu}, Lemma G.6] \label{thm:lassol1}
    Assume that  $\phi^2(\mM, s) > 0$. The Lasso estimator $\tinit$ in~\eqref{eq:lassoinit} satisfies
    \begin{align}
        \Prob{\norm{\tinit - \ts}_1 \leq \frac{4 \lambda_{\mr{init}} s}{\phi^2(\mM, s)}} \geq 1 - 2 d \exp \left\{ - \frac{T_1 \lambda_{\mr{init}}^2}{32 \left( \frac{1}{T_1} \max_{j \in [d]} \norm{\mX_{:, j}}_2^2 \right)} \right\}.
    \end{align}
\end{lemma}

Using \lemref{thm:lassol1}, we derive the following bound on the event that the size of the support of the TL output~\eqref{eq:threshlas} is below a threshold and it captures the true support $S(\ts)$. 
\begin{theorem}\label{thm:thres}
    Fix a design matrix $\mX \in \mathbb{R}^{T_1 \times d}$ and parameters $\lambda_{\mathrm{init}}, \lambda_{\mathrm{thres}} > 0$.  Let $b = \frac{4}{\phi^2(\mM, s)}$ and $c = \frac{\lambda_{\mathrm{thres}}}{\lambda_{\mathrm{init}}}$. Suppose that $\theta_{\min} \geq \lambda_{\mathrm{init}} \left( c + b s \right)$ holds. Then,
    \begin{align}
        \Prob{ \left\{ |S(\tthres)| \leq s \left( 1 + \frac{b}{c} \right) \right \} \bigcap \{ S(\tthres) \supseteq S(\ts) \} }\geq 1 - 2d  \exp \left\{ - \frac{T_1 \lambda_{\mathrm{init}}^2}{32 \left( \frac{1}{T_1} \max_{j \in [d]} \norm{\mX_{:, j}}_2^2 \right)} \right\}. \label{eq:cor}
    \end{align}
\end{theorem}
The proofs of \lemref{thm:lassol1} and \thmref{thm:thres} are deferred to Appendix~\ref{app:lasso}. Theorem~\ref{thm:thres} follows steps similar to those in \citet[Lemma 5.4]{ariu}. The interested reader can refer to \citet[Ch.~6 and~7]{vandegeerbook} for more results and discussions on Lasso, TL, and their variants.

\subsection{An Improved Upper Bound on the Error Probability of OD-LinBAI}
The theorem below gives an improved upper bound on the error probability of OD-LinBAI \citep{yang}. 
\begin{theorem} \label{thm:ODLinBAI}
    Let $\tilde{T} = \left \lfloor \frac{T}{\lceil \log_2 d \rceil} \right \rfloor$. For any linear bandit instance, the output of OD-LinBAI satisfies
    \begin{align}
        \Prob{\hat{I} \neq 1} \leq (K + \log_2 d) \exp \left\{ - \frac{\tilde{T}}{16 \left( 1 + \frac{d^2}{\tilde{T}}\right)  H_{2, \mathrm{lin}}(d)} \right\}. \label{eq:ODLinBAI}
    \end{align}
\end{theorem}
The right-hand side of \eqref{eq:ODLinBAI} is slightly different than the one presented in \citet[Th.~2]{yang}. First, in \citet[Th.~2]{yang}, the numerator in the exponent is equal to some constant $m$ that is approximately equal to $\frac{T}{\log_2 d}$ just like $\tilde{T}$; this is due to the modification in the distribution rounding technique. 
Second, the pre-factor in \citet[Th.~2]{yang} is $\frac{4K}{d} + 3 \log_2 d$ instead of (our smaller) $K + \log_2 d$. More importantly, in \eqref{eq:ODLinBAI}, the constant $32$ in the denominator of the exponent in \citet[Th.~2]{yang} is improved to~$16$. The last two differences are due to a refinement in the proof technique. Lastly, our result includes a rounding error factor $1 + \frac{d^2}{\tilde{T}}$, which becomes negligible as $T$ becomes large. This factor appears due to the fact that the G-optimal design may yield fractional number of pulls for some arms, which is obviously not allowed in practice. The proof of Theorem~\ref{thm:ODLinBAI} is deferred to Appendix~\ref{app:proofOD}.

\subsection{Upper Bound on the Error Probability of Lasso-OD}
The theorem below bounds the probability of incorrectly identifying the best arm using Lasso-OD.
\begin{theorem}\label{thm:main}
Let $T_1 < T$ be the length of phase 1, and let $T_2 = T - T_1$ be the length of phase~2. Let $\lambda_{\mr{init}}$ and $\lambda_{\mr{thres}}$ be some positive scalars. Let $c = \frac{\lambda_{\mr{thres}}}{\lambda_{\mr{init}}}$. Let $\tilde{\nu}^*$ be the solution to \eqref{eq:vtilde}, and let $\tilde{\nu} = \mathrm{ROUND}(\tilde{\nu}^*, T_1)$ be its rounded version for length $T_1$. Suppose that $b = \frac{4}{\phi^2(\tilde{\nu}, s)} > 0$ and $\theta_{\min} \geq \lambda_{\mr{init}} (c + bs)$. For any linear bandit instance, the output of Algorithm~\ref{alg:lasso-OD} satisfies
\begin{align}
    \Prob{\hat{I} \neq 1} \leq (K + \log_2 d) \exp\left\{ - \frac{ \left \lfloor \frac{T_2}{\log_2(s_1)} \right \rfloor}{16 \left(1 + \epsilon \right) H_{2, \mr{lin}} (s_1) } \right\} + 2d \exp \left\{ - \frac{T_1 \lambda_{\mr{init}}^2}{32 x_{\max}^2}  \right\}, \label{eq:maineq}
\end{align}
where
\begin{align}
    s_1 &= \left \lfloor s \left(1 + \frac{b}{c} \right) \right \rfloor, \quad
    x_{\max}^2 = \max \limits_{j \in [d]}  \sum_{k = 1}^K \tilde{\nu}_k (\va(k)_j)^2, \quad\mbox{and}\quad
    \epsilon = \frac{s_1^2}{T_2}. \label{eq:s1}
\end{align}
\end{theorem}
\begin{proof}
The proof uses Theorems~\ref{thm:thres} and \ref{thm:ODLinBAI} for the probability terms due TL and OD-LinBAI, respectively. 
Let $\hat{\mc{S}} \subseteq [d]$ denote the output of phase 1. Define the events $\mc{E} \triangleq \{ |\hat{\mc{S}}| \leq s_1\}$ and $\mc{F} \triangleq \{ \hat{\mc{S}} \supseteq S(\ts) \}$. By the law of total probability, we have
\begin{align}
    \Prob{\hat{I} \neq 1} 
    &\leq  \Prob{\hat{I} \neq 1 \middle | \mc{E} \cap \mc{F}} + \Prob{\mc{E}^{\mr{c}} \cup \mc{F}^{\mr{c}}}. \label{eq:Icond}
\end{align}
Given $\mc{E} \cap \mc{F}$, the error probability is bounded by the right-hand side of \eqref{eq:ODLinBAI} with the budget $T$ replaced by the length of phase 2, $T_2$, and with the dimension $d$ replaced by $s_1$. This follows since on the event $\mc{F}$, the mean rewards are preserved after the arm vectors and $\ts$ are projected on $\hat{\mc{S}}$ and since the right-hand side of \eqref{eq:ODLinBAI} is non-decreasing in $d$. From Theorem~\ref{thm:thres} and the arm-pulling strategy described in Line 2 of Algorithm~\ref{alg:lasso-OD}, we have
\begin{align}
    \Prob{\mc{E}^{\mr{c}} \cup \mc{F}^{\mr{c}}} \leq 2d \exp \left\{ - \frac{T_1 \lambda_{\mr{init}}^2}{32 x_{\max}^2}  \right\}. \label{eq:lassoapplied}
\end{align}
Combining \eqref{eq:Icond} with \eqref{eq:ODLinBAI} and \eqref{eq:lassoapplied}, we complete the proof. 
\end{proof}
The following corollary is obtained by choosing the free parameters $T_1, \lambda_{\mr{init}}$, and $\lambda_{\mr{thres}}$ suitably to meet the conditions of Theorem~\ref{thm:main}. These nontrivial choices use the knowledge of $\theta_{\min}$ and $s$ but not the hardness parameter and achieve an exponent of the error probability that depends only on $s$, $T$, and the hardness parameter.
\begin{corollary} \label{cor:main}
    For any linear bandit instance, it holds that
    \begin{align}
        \Prob{\hat{I} \neq 1} \leq (K + \log_2 d + 2d) \exp \left \{ - \frac{T}{16 \lfloor\log_2(s + s^2) \rfloor (1 + \epsilon) H_{2, \mr{lin}}(s + s^2) (1 + c_0)} \right \},
    \end{align}
    where
    \begin{align}
        c_0 &= \frac{25 b^2 x_{\max}^2}{3 \theta_{\min}^2 \log_2(s + s^2)}\quad \mbox{and}\quad
        \epsilon = \frac{(1 + c_0)(s + s^2)^2}{T}. \label{eq:c0}
    \end{align}
\end{corollary}

Here, $c_0 = \frac{T_1}{T_2}$ is the fraction of lengths of two phases of Lasso-OD, and $1 + \epsilon$ is the penalty due to rounding. Since $c_0$ in~\eqref{eq:c0} is lower bounded by a positive constant for all $s \in \mathbb{N}$,
Corollary~\ref{cor:main} implies that the error probability of Lasso-OD is upper bounded by
\begin{align}
    \exp \left\{ - \Omega \left(\frac{T}{ (\log_2 s) H_{2, \mr{lin}}(s + s^2)} \right)\right\} \label{eq:upper}
\end{align}
for fixed $s$, $T \to \infty$, and $K$ and $d$ not growing exponentially with $T$. Therefore, unlike the non-sparse case in \cite{yang}, the error probability exponent is {\em independent of the dimension $d$}, but instead, depends on the sparsity $s$, which yields much smaller error probabilities for high dimensional sparse linear bandits. The parameter choices that achieve the exponent in \eqref{eq:upper} is nontrivial; we carefully choose $\lambda_{\mr{init}}$ and $\lambda_{\mr{thres}}$ so that $c_0$ is decreasing in $s$ and choose $T_1$ so that two exponents in~\eqref{eq:maineq} emanating from phases 1 and 2 are ``balanced''. The proof of Corollary~\ref{cor:main} is presented in Appendix~\ref{app:maincor}.

Assume that the agent knows the support of $\ts$. Then, following the construction in the proof of \citet[Th.~3]{yang}, for any algorithm, there exists a bandit instance whose error probability is lower bounded by $ \exp \big\{ - O\big(\frac{T}{ (\log_2 s) H_{2, \mr{lin}}(s)} \big)\big\}$. This implies that the upper bound in~\eqref{eq:upper} is indeed almost minimax optimal in the exponent. In Appendix~\ref{app:experiments}, we develop
a variant of Lasso-OD, called PopArt-OD, which replaces TL in phase 1 of Lasso-OD with PopArt from \citet{jang2022}. Thanks to the fact that PopArt provides a guarantee on the $\ell_\infty$ norm of the difference between the estimated parameter $\bm{\theta}'$ and $\ts$, we derive an upper bound on the probability $\Prob{\hat{\mc{S}}_{\mr{PA}} \neq S(\ts)}$, where $\hat{\mc{S}}_{\mr{PA}}$ denotes the estimated support using the PopArt algorithm. Using this bound, we show that the error probability of PopArt-OD is upper bounded by $ \exp \big\{ - O\big(\frac{T}{ (\log_2 s) H_{2, \mr{lin}}(s)} \big)\big\}$, matching the lower bound up to a constant factor in the exponent. Due to its superior empirical performance over PopArt-OD, we focus on Lasso-OD in the paper.

 \section{Experiments} \label{sec:experiments}
 In this section, we numerically evaluate the performance of Lasso-OD on several synthetic sparse linear bandit instances and compare it with those of OD-LinBAI \citep{yang}, BayesGap \citep{hoffman}, GSE \citep{azizi}, Peace \citep{katzsamuels}, and LinearExploration \citep{alieva}. In each setting, we report the empirical error probabilities for Lasso-OD, BayesGap, and GSE over 4000 independent trials and for Peace and LinearExploration over 100 independent trials.
\begin{figure}[t]
\includegraphics[width=0.87\textwidth]{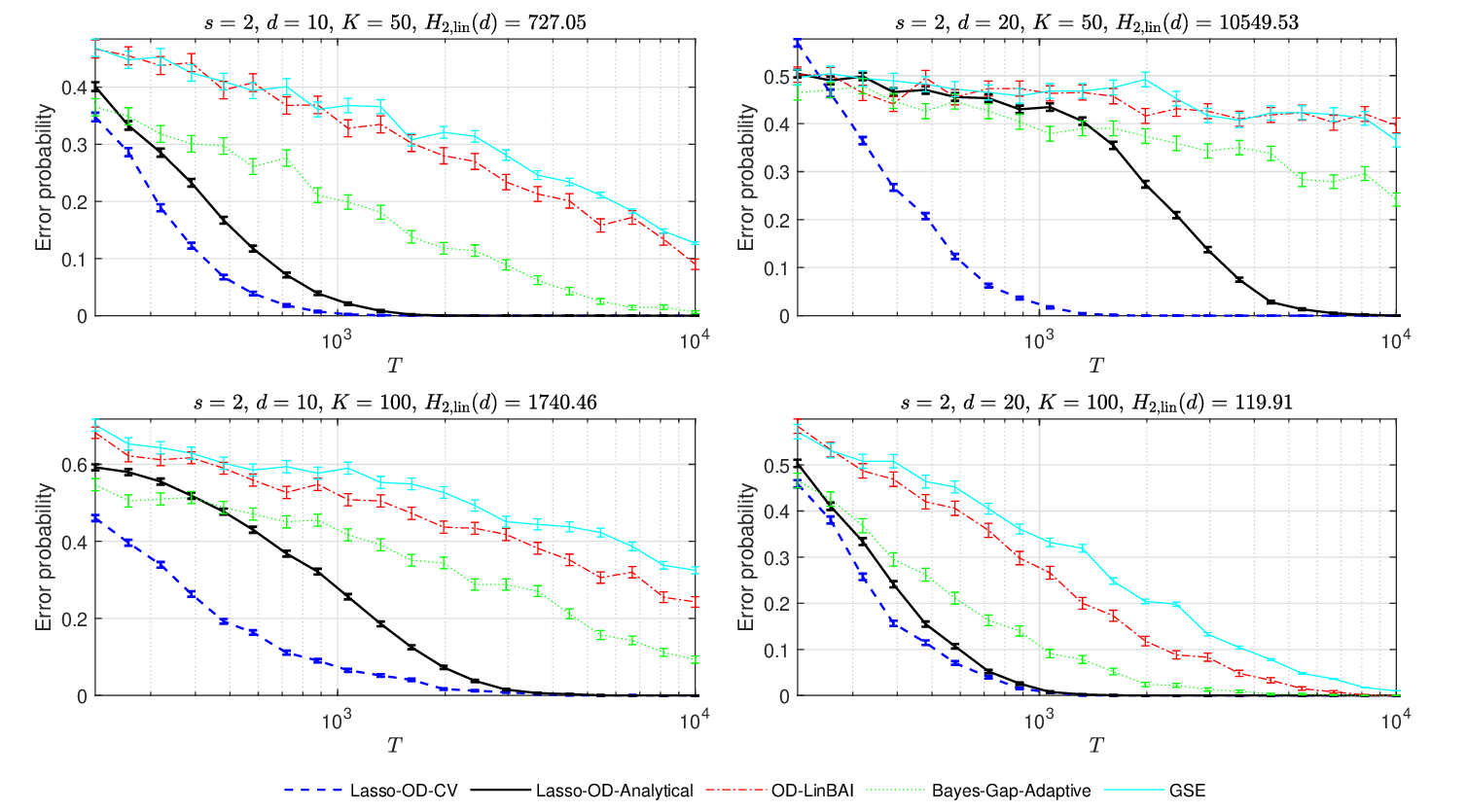}
\centering
\caption{Comparison of several algorithms with $T \in [200, 10000]$ and $s = 2$.} \vspace{-.1in}
\label{fig:plots}
\end{figure}
 \subsection{Synthetic Sparse Dataset} 
 In the first example, we draw $K$ arms independently from the uniform distribution on the $d$-dimensional sphere of radius $\sqrt{d / s}$, i.e., $\big\{x \in \mathbb{R}^d \colon \norm{x}_2^2 = \frac{d}{s} \big\}$, and the sparse unknown parameter is taken as $\ts = (1, 1, 0, \dots, 0)$, i.e., $s = 2$. Figure~\ref{fig:plots} reports the empirical error probabilities for $d \in \{10, 20\}$, $K \in \{50, 100\}$ and $T \in [200, 10000]$, except Peace \citep{katzsamuels} and LinearExploration \citep{alieva}. Since the computational complexities of Peace and LinearExploration are much higher than the rest of the algorithms, we compare Lasso-OD with Peace and LinearExploration only for $T = 800$ in Table~\ref{tab:800}. Among these algorithms, Lasso-OD has the best performance for all sparse instances shown in  Figure~\ref{fig:plots} and Table~\ref{tab:800}.
 \begin{table}[t]
 \caption{Performance comparison of several algorithms for $T = 800$, $d = 10$, $K =50$, and $s = 2$.}
 \centering
\begin{tabular}{p{3cm} 
p{2.2cm}p{2.2cm}p{2.2cm}p{2.2cm}}
\toprule
\multicolumn{1}{l}{} & Lasso-OD-CV & Lasso-OD-An. & Peace & LinearExploration \\ \midrule
Error probability    & 0.0275   & 0.045        & 0.40  & 0.39              \\ 
Std. deviation   & 0.0026   & 0.0033       & 0.049 & 0.0153            \\ \bottomrule
\end{tabular}
\label{tab:800}
\end{table}

\begin{figure}[t]
\includegraphics[width=0.87\textwidth]{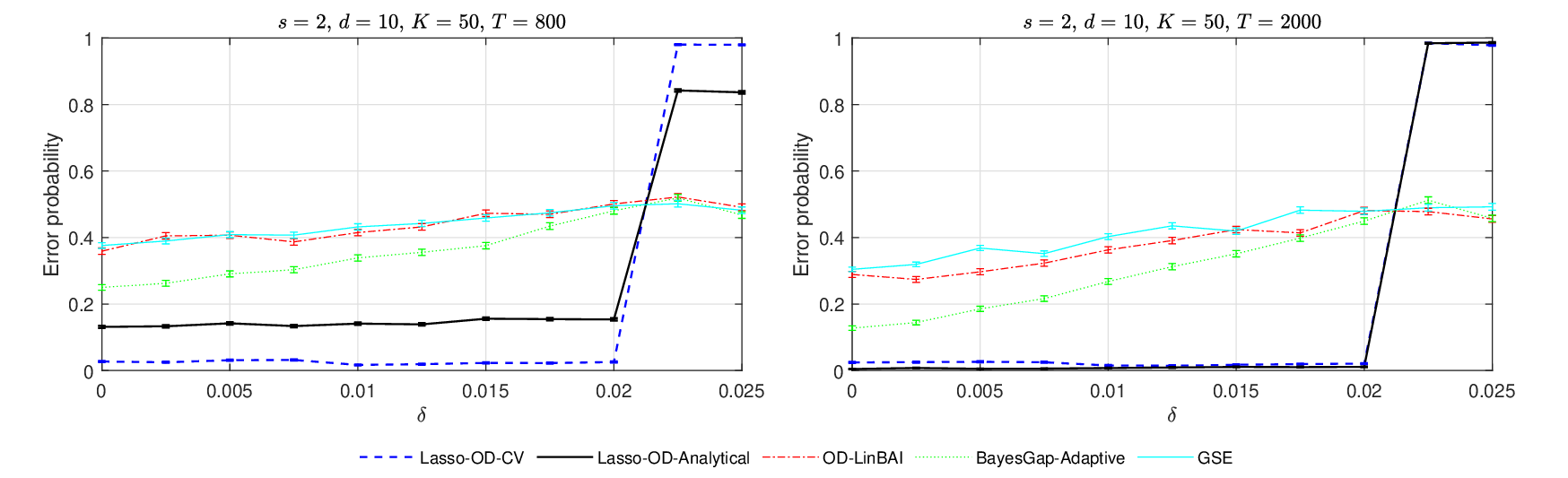}
\centering
\caption{Comparison of several algorithms with $T \in \{800, 2000\}$, $s = 2$, and $\delta \in [0, 0.025]$.} \vspace{-.1in}
\label{fig:robust}
\end{figure}
 
 Lasso-OD-CV sets the budgets for phase 1 and phase 2 as $T_1 = \frac{T}{5}$ and $T_2 = \frac{4 T}{5}$ and tunes the Lasso parameters $\lambda_{\mr{init}}$ and $\lambda_{\mr{thres}}$ using a $K$-fold cross-validation procedure that uses the value of $s$ in its loss function. See~\appref{app:CV} for the details of the cross-validation procedure. As an alternative to cross-validation, Lasso-OD-Analytical uses the knowledge of $s$, $\theta_{\min}$, and the hardness parameter $H_{2, \mr{lin}}(s_1)$ in~\eqref{eq:maineq}, and sets  $\lambda_{\mr{init}}$, $\lambda_{\mr{thres}}$, and $T_1$ so that $s_1 $ in~\eqref{eq:s1} equals $s + s^2$, $\theta_{\min} = \lambda_{\mr{init}} (c + bs)$, and two exponents in~\eqref{eq:maineq} are equal. Note that $H_{2, \mr{lin}}(s_1)$ is usually not available to the agent.

In the second example, we test the robustness of our algorithm with respect to the variables in $\ts$ that are assumed to be zero by keeping the same arms as in the previous example and setting $\ts$ as $\theta^*_j = 1$ for $j \in [2]$, and $\theta^*_j = \delta R_j$ for $j \in \{3, \dots, d\}$, where $R_j$, $j = 3, \dots, d$, are independent Rademacher  (i.e., $\{\pm 1\}$-valued) random variables, and $\delta > 0$ is a constant.  Figure~\ref{fig:robust} reports the empirical error probabilities for $s = 2$, $d = 10$, $K = 50$, $T \in \{800, 2000\}$, and $\delta \in [0, 0.025]$. The phase transition for Lasso-OD in Figure~\ref{fig:robust} suggests that Lasso-OD achieves a smaller error probability as long as $\delta$ is small enough that the approximately sparse instance (i.e., $\delta > 0$) and the sparse instance (i.e., $\delta = 0$) have the same best arm. Some examples including an instance where the hyperparameters are set as in \corref{cor:main} without cross-validation or knowing the hardness parameter are discussed in \appref{app:experiments}.

 \section{Conclusion}
 In this work, we study the BAI problem in linear bandits with sparse structure under fixed-budget setting and develop the first BAI algorithm, Lasso-OD, that exploits the sparsity of the unknown parameter $\ts$. Lasso-OD combines TL for support estimation with the minimax optimal BAI algorithm, OD-LinBAI. We analyze the error probability of Lasso-OD and show that the error exponent depends on the sparsity $s$ rather than the dimension $d$. Unlike other algorithms in the literature, the empirical performance of Lasso-OD does not deteriorate at large dimensions.

  One future direction is to derive an instance-dependent asymptotic or non-asymptotic lower bound for the BAI problem in sparse linear bandits; however, such a bound remains open even in the non-sparse scenario.  Another possible direction is to extend the TL technique used in Lasso-OD to the fixed-confidence setting. Although such an extension is relatively easy to analyze, the empirical performances of most fixed-confidence BAI algorithms in linear bandits are not heavily dependent on the dimension unlike the fixed-budget setting (see, for example, \citet{peleg, tao, fiez}). Therefore, the benefit of adding a TL phase in the fixed-confidence setting could be limited.

\bibliography{tmlr}
\bibliographystyle{tmlr}

\appendix
\onecolumn
\section{Design Matrix Optimization in Thresholded Lasso}\label{app:design}

The performance of TL in Lasso-OD is characterized by the probability term in \eqref{eq:cor}. If we relax the quantity $\frac{1}{T_1} \max_{j \in [d]} \norm{\mX_{:, j}}_2^2$ in \eqref{eq:cor} by its upper bound $\max_{k \in [K]} \norm{\va(k)}_{\infty}^2$, \thmref{thm:thres} implies that the performance of TL depends on the design matrix $\mX$ through $b$, and the best choice of $\mX$ maximizes the compatibility constant $\phi^2(\mM, s)$.

\paragraph{Computation of the compatibility constant} Let $\nu \in \mathcal{P}_{T_1}([K])$ be the $T_1$-type distribution describing the fractions of the number of pulls for each arm. Then, $\mM = \frac{1}{T_1} \mX^\top \mX = \sum_{i \in [K]} \nu_i \va(i) \va(i)^\top$. Rewriting the compatibility constant $\phi^2(\mM, \mathcal{S})$ from Definition~\ref{def:comp}, with some overload of notation, we obtain
\begin{align}
    \phi^2(\nu, \mathcal{S}) &\triangleq \phi^2(\mM, \mathcal{S}) = \min_{\vtheta \in \mathbb{R}^d} \left\{ |\mathcal{S}| \norm{\vtheta}_{\sum_{i \in [K]} \nu_i \va(i) \va(i)^\top}^2   \colon \norm{\vtheta_{\mathcal{S}}}_1 = 1,  \norm{\vtheta_{\mathcal{S}^{\mathrm{c}}}}_1 \leq 3 \right \}\label{eq:phi2} \\
    \phi^2(\nu, s) &\triangleq \min_{\mathcal{S} \subseteq [d] \colon |\mathcal{S}| = s} \phi^2(\nu, \mathcal{S}). \label{eq:phinu}
\end{align}
Given a fixed $\nu$, the program in \eqref{eq:phi2} is non-convex due to the $\ell_1$-norm equality constraint; however, by introducing binary variables, it can be turned into a mixed-integer discipled convex program (MIDCP) and be efficiently solved using CVX toolbox \citep{boyd}. If we relaxed the equality constraint to $\norm{\vtheta_{\mathcal{S}}}_1 \leq 1$, then \eqref{eq:phi2} would be a quadratic program~(QP). 

\paragraph{Relaxing the optimization problem}
According to the arguments above, the optimization problem that we originally need to solve is
\begin{align}
    \nu^* = \argmax_{\nu \in \mathcal{P}_{T_1}([K])} \phi^2(\nu, s), \label{eq:vstar}
\end{align}
which is computationally intractable since the maximization constraint makes it an integer program; and even if we relaxed it to allow fractional number of pulls, the program would involve $\binom{d}{s} \approx d^s$ MIDCPs in its constraints. 
\begin{lemma} \label{lem:sigma}
For any $\nu \in \mathcal{P}([K])$ and any $\mathcal{S} \subseteq d$, it holds that $\phi^2(\nu, \mathcal{S}) \geq \sigma_{\min}(\sum_{i = 1}^K \nu_i \va(i) \va(i)^\top)$. 
\end{lemma}
Lemma~\ref{lem:sigma} follows from $|\mathcal{S}| \norm{\vtheta{\mathcal{S}}}_1^2 \leq \norm{\vtheta{\mathcal{S}}}_2^2 \leq \norm{\theta}_2^2$ and relaxing the inequality constraint in \eqref{eq:phi2}. 

Replacing $\phi^2(\nu, \mathcal{S})$ by its lower bound and allowing fractional number of pulls, we get the relaxed optimization problem in \eqref{eq:vtilde}, which can be solved efficiently.

\section{Pseudo-codes of ROUND and OD-LinBAI} \label{app:algorithms}
The pseudo-codes of the rounding procedure from \citet[Ch.~12]{pukel} that is used in Algorithm~\ref{alg:ESP} and OD-Lasso from \citet{yang} are given below.
\label{app:ODLinBAI}
\begin{algorithm}
    \caption{ROUND$(\pi, T)$}
   \label{alg:round}
\begin{algorithmic}[1]
\INPUT a distribution $\pi$ on a set with cardinality $d$ and a positive integer $T$.
\STATE Initialize $T_i = \lceil (T - \frac{d}{2}) \pi_i \rceil$ for $i = 1, \dots, d$. 
\WHILE{ $\sum_{i = 1}^d T_i \neq T$ }
\IF{ $\sum_{i = 1}^d T_i < T$ }
\STATE Set $j \gets \arg \min_{i \in [d]} \frac{T_i}{\pi_i}$. Update $T_j \gets T_j + 1$. 
\ELSIF{  $\sum_{i = 1}^d T_i > T$}
\STATE Set $j \gets \arg \max_{i \in [d]} \frac{T_i - 1}{\pi_i}$. Update $T_j \gets T_j - 1$. 
\ENDIF
\ENDWHILE
\OUTPUT Distribution $\tilde{\pi} = \left( \frac{T_1}{T}, \dots, \frac{T_d}{T} \right)$.
\end{algorithmic}
\end{algorithm}

\begin{algorithm}[!htbp]
   \caption{Optimal Design-Based Linear Best Arm Identification (OD-LinBAI)}
   \label{alg:ODLINBAI}
\begin{algorithmic}[1]
   \INPUT time budget $T$, arm set $\mathcal{A} = [K]$, and arm vectors $\{\va(1), \dots, \va(K)\} \in \mathbb{R}^d$.
   \STATE Initialize $t_0 \gets 0$,  $\mathcal{A}_0 \gets \mathcal{A}$, $d_0 \gets d$. For each $i \in \mathcal{A}_0$, set $\va_0(i) = \va(i)$. Set $R = \lceil \log_2 d \rceil$, $T_r = \left \lfloor \frac{T}{R} \right \rfloor$ for $r = 1, \dots, R - 1$, and $T_R = T - \sum_{i = 1}^{R-1} T_i$.

   \FOR{$r=1$ {\bfseries to} $R$}
   \STATE $\backslash \backslash$ Dimensionality reduction:
   \STATE Set $\mX$ so that its columns are $\{ \va_{r-1}(i) \colon i \in \mathcal{A}_{r-1}\}$. Set $d_r \gets \mathrm{rank}(\mX)$. Set $\va_r(i) \gets \va_{r-1}(i)$ for $i \in \mathcal{A}_{r-1}$.
   \IF{$d_r < d_{r-1}$}
   \STATE Find the singular value decomposition of $\mX = \mU \mD {\mV}^\top$, where $\mU \in \mathbb{R}^{d_{r-1} \times d_r}$. 
   \STATE Update $\mX \gets \mU^\top \mX$ and $\va_r(i) \gets \mX_i$ for $i \in \mathcal{A}_{r-1}$.
    \ENDIF
    \STATE $\backslash \backslash$ G-optimal design:
    \STATE Input the set $\{\va_r(i) \colon i \in \mathcal{A}_{r-1}\}$ to the G-optimal design, and set $\pi^{(r)}$ as the output of~\eqref{eq:goptimal}. 
    \STATE Set $\tilde{\pi}^{(r)} = \textrm{ROUND}(\pi^{(r)}, T_r)$ from Algorithm 2.
    \STATE $\backslash \backslash$ Arm pulling:
    \STATE Pull each arm $i \in \mathcal{A}_{r-1}$ $T_r(i) = \tilde{\pi}^{(r)}_i T_r$ times, which determines $A_{t_{r-1} + 1}, \dots, A_{t_{r-1} + T_r}$. Observe the corresponding rewards $\ry_{t_{r-1} + 1}, \dots, \ry_{t_{r-1} + T_r}$.
    \STATE Compute the OLS estimator
    \begin{align}
        \mV^{(r)} &= \sum_{i \in \mathcal{A}_{r-1}} T_r(i) \va_r(i) \va_r(i)^{\top} \\
        \that^{(r)} &= {\mV^{(r)}}^{-1} \sum_{t = t_{r-1} + 1}^{t_{r-1} + T_r} \va_r(A_t) \ry_t.
    \end{align}
    \STATE $\backslash \backslash$ Arm elimination:
    \STATE Estimate the mean rewards for each $i \in \mathcal{A}_{r-1}$ as
    \begin{align}
        \hat{\mu}_r(i) = \langle \that^{(r)}, \va_r(i) \rangle. 
    \end{align}
    Set $\mathcal{A}_{r} \gets $ the set of $\lceil \frac{d}{2^r} \rceil$ arms in $\mathcal{A}_{r-1}$ with the largest estimated mean rewards. 
    Set $t_r \gets t_{r-1} + T_r$. 
   \ENDFOR
    \OUTPUT $\hat{I} = $ the only remaining arm in $\mathcal{A}_R$. 
   
\end{algorithmic}
\end{algorithm}

\section{Proofs Related to Lasso}\label{app:lasso}
In the following, let $n$ be the number of samples. The linear model is given by $\rvy = \mX \ts + \bm{\epsilon}$, where $\rvy \in \mathbb{R}^n$ are the rewards, $\mX \in \mathbb{R}^{n \times d}$ is the design matrix, and $\bm{\epsilon} \in \mathbb{R}^n$ are i.i.d. 1-subgaussian random variables. Recall the initial Lasso estimator
\begin{align}
    \that = \argmin_{\vtheta \in \mathbb{R}^d} \frac{1}{n} \norm{\rvy - \mX \vtheta}_2^2 + \lambda \norm{\vtheta}_1. \label{eq:lasso}
\end{align}

Define the event
\begin{align}
    \mathcal{T} = \left\{\max_{j \in [d]} \frac{1}{n} |\mX_{:, j}^\top \bm{\epsilon}| \leq \frac{\lambda}{4} \right\}.
\end{align}
The following result, known as the oracle inequality, is the main tool to control the performance of the initial lasso estimator. 
\begin{lemma}[Oracle Inequality: Theorem 6.1 from \cite{vandegeerbook}]\label{lem:oracle}
    On the event $\mathcal{T}$, the initial Lasso estimator $\that$ \eqref{eq:lasso} satisfies
    \begin{align}
        \norm{\mX(\that - \ts)}_2^2 + \lambda \norm{\that - \ts}_1 \leq \frac{4 \lambda^2 s}{\phi^2(\mM, S(\ts))}. \label{eq:oracle}
    \end{align}
    Furthermore, it holds that
    \begin{align}
        \Prob{\mathcal{T}} \geq 1 - 2 d \exp \left\{ - \frac{n \lambda^2}{32 \left( \frac{1}{n} \max_{j \in [d]} \norm{\mX_{:, j}}_2^2 \right)} \right\}. \label{eq:T}
    \end{align}
    
\end{lemma}

\begin{proof}[Proof of Lemma~\ref{lem:oracle}] 
Since $\that$ minimizes \eqref{eq:lasso}, we have
\begin{align}
    \frac{1}{n} \norm{ \rvy - \mX \that}_2^2 + \lambda \norm{\that}_1 \leq \frac{1}{n} \norm{ \rvy - \mX \ts}_2^2 + \lambda \norm{\ts}_1. \label{eq:mainineq}
\end{align}
Plugging $\rvy = \mX \ts + \bm{\epsilon}$ into \eqref{eq:mainineq}, after some algebra, we get the basic inequality
\begin{align} 
    \frac{1}{n} \norm{ \mX (\that - \ts)}_2^2 + \lambda \norm{\that}_1 \leq \frac{2}{n} \bm{\epsilon}^\top \mX (\that - \ts) + \lambda \norm{\ts}_1.
\end{align}
Let $\tilde{\mc{T}}$ be the event
\begin{align}
    \tilde{\mc{T}} = \left \{ \max_{j \in [d]} \frac{2}{n} |\bm{\epsilon}^\top \mX_{:, j} | \leq \lambda_0 \right\}.
\end{align}
Then, on $\tilde{\mc{T}}$, we have using the H\"older inequality that
\begin{align}
    \frac{1}{n} \norm{ \mX (\that - \ts)}_2^2  \leq \lambda_0 \norm{\that - \ts}_1 + \lambda{\norm{\ts}}_1 - \lambda \norm{\that}_1. \label{eq:X1}
\end{align}
Let $\mc{S} = S(\ts)$. By the triangle inequality, we have
\begin{align}
    \norm{\that}_1 = \norm{\that_{\mc{S}}}_1 + \norm{\that_{\mc{S}^{\mr{c}}}}_1 \geq \norm{\ts_{\mc{S}}}_1 - \norm{\that_{\mc{S}} - \ts_{\mc{S}}}_1 + \norm{\that_{\mc{S}^{\mr{c}}}}_1. \label{eq:holderused}
\end{align}
Applying \eqref{eq:holderused} to \eqref{eq:X1}, we get
\begin{align}
     \frac{1}{n} \norm{ \mX (\that - \ts)}_2^2 &\leq \lambda_0 \left( \norm{\that_{\mc{S}} - \ts_{\mc{S}}}_1 + \norm{\that_{\mc{S}^{\mr{c}}} - \ts_{\mc{S}^{\mr{c}}}}_1 \right) \notag \\
     &\quad + \lambda \left( \norm{\ts}_1 - \norm{\ts_{\mc{S}}}_1 + \norm{\that_{\mc{S}} - \ts_{\mc{S}}}_1 - \norm{\that_{\mc{S}^{\mr{c}}}}_1 \right) \\
     &= (\lambda_0 + \lambda) \norm{\that_{\mc{S}} - \ts_{\mc{S}}}_1 + (\lambda_0 - \lambda) \norm{\that_{\mc{S}^{\mr{c}}} - \ts_{\mc{S}^{\mr{c}}}}_1, \label{eq:lambdaineq}
\end{align}
where the last step uses the fact that $\ts_{\mc{S}^{\mr{c}}} = 0$. We set $\lambda_0 = \frac{\lambda}{2}$. Then, \eqref{eq:lambdaineq} implies that on the event $\mc{T}$, 
\begin{align}
    \norm{\that_{\mc{S}^{\mr{c}}} - \ts_{\mc{S}^{\mr{c}}}}_1 \leq 3 \norm{\that_{\mc{S}} - \ts_{\mc{S}}}_1.
\end{align}
Therefore, $\that - \ts \in \mathbb{C}(\mc{S})$, and from the definition of compatibility constant in Definition~\ref{def:comp}, we have
\begin{align}
    \norm{\that_{\mc{S}} - \ts_{\mc{S}}}_1 \leq \frac{  \sqrt{ s (\that - \ts)^\top \mX^\top \mX (\that - \ts) }}{\sqrt{n} \phi(\mM, \mc{S})}. \label{eq:compused}
\end{align}

We now continue with \eqref{eq:lambdaineq} with $\lambda_0 = \frac{\lambda}{2}$. We have
\begin{align}
    \frac{2}{n} \norm{ \mX (\that - \ts)}_2^2 + \lambda \norm{\that - \ts}_1 &= \frac{2}{n} \norm{ \mX (\that - \ts)}_2^2 + \lambda \norm{\that_{\mc{S}} - \ts_{\mc{S}}}_1 + \lambda \norm{\that_{\mc{S}^{\mr{c}}} - \ts_{\mc{S}^{\mr{c}}}}_1 \\
    &\leq (3 \lambda + \lambda) \norm{\that_{\mc{S}} - \ts_{\mc{S}}}_1 - \lambda \norm{\that_{\mc{S}^{\mr{c}}} - \ts_{\mc{S}^{\mr{c}}}}_1 + \lambda \norm{\that_{\mc{S}^{\mr{c}}} - \ts_{\mc{S}^{\mr{c}}}}_1 \\
    &= 4 \lambda \norm{\that_{\mc{S}} - \ts_{\mc{S}}}_1 \\
    &\leq \frac{4 \lambda \sqrt{s} \norm{\mX (\that - \ts)}_2}{\sqrt{n} \phi(\mM, \mc{S})} \label{eq:compstep} \\
    &\leq \frac{1}{n} \norm{\mX (\that - \ts)}_2^2  + \frac{4 \lambda^2 s}{\phi^2(\mM, \mc{S})}, \label{eq:u2}
\end{align}
where \eqref{eq:compstep} applies \eqref{eq:compused}, and \eqref{eq:u2} applies the inequality $4 uv \leq u^2 + 4 v^2$ to \eqref{eq:compstep}. Inequality \eqref{eq:u2} completes the proof of \eqref{eq:oracle}.

Next, we upper bound the probability $\Prob{\tilde{\mc{T}}^{\mr{c}}}$. We have
\begin{align}
    \Prob{\tilde{\mc{T}}^{\mr{c}}} &= \Prob{\bigcup_{j \in [d]} \frac{1}{n} |\mX_{:, j}^\top \bm{\epsilon}| > \frac{\lambda}{4}} \\
    &\leq \sum_{j = 1}^d \left(\Prob{ \frac{1}{n} \mX_{:, j}^\top \bm{\epsilon} > \frac{\lambda}{4}} + \Prob{ - \frac{1}{n} \mX_{:, j}^\top \bm{\epsilon} > \frac{\lambda}{4}} \right) \\
    &\leq 2 \sum_{j = 1}^d \exp \left \{ - \frac{\lambda^2}{2 \cdot 4^2 \cdot \frac{1}{n^2} \norm{\mX_{:, j}}_2^2} \right \}, \label{eq:Tc} 
\end{align}
where the last inequality follows since $\frac{1}{n} \mX_{:, j}^\top \bm{\epsilon}$ $-\frac{1}{n} \mX_{:, j}^\top \bm{\epsilon}$ are subgaussian with variance proxy $\frac{1}{n^2} \norm{\mX_{:, j}}_2^2$ as $\epsilon_1, \dots, \epsilon_n$ are independent 1-subgaussian random variables. Bounding each summand in \eqref{eq:Tc} by the maximum of summands completes the proof of \eqref{eq:T}.
\end{proof}

\begin{proof}[Proof of \lemref{thm:lassol1}]
The right-hand side of \eqref{eq:oracle} depends on the unknown set $S(\ts)$; however, we can further upper bound the right-hand side of \eqref{eq:oracle} by replacing $\phi^2(\mM, S(\ts))$ by its lower bound $\phi^2(\mM, s)$, which is computable using only $\mX$ and $s$. Therefore, \lemref{thm:lassol1} is a corollary to Lemma~\ref{lem:oracle}.
\end{proof}

\begin{proof}[Proof of \thmref{thm:thres}]
Define the event $\mc{G} \triangleq \left\{ \norm{\tinit - \ts}_1 \leq \frac{4 \lambda_{\mr{init}} s}{\phi^2(\mM, s)} \right\}$. We have
\begin{align}
    \norm{\tinit - \ts}_1 &\geq \norm{(\tinit - \ts)_{S(\ts)^{\mr{c}}}}_1 \\
    &= \sum_{j \in S(\ts)^{\mr{c}}} |(\tinit)_j | \label{eq:tsstep}\\
    &\geq \sum_{j \in S(\tthres) \setminus S(\ts)} |(\tinit)_j | \\
    &\geq | S(\tthres) \setminus S(\ts) | \lambda_{\mr{thres}}, \label{eq:thresstep}
\end{align}
where \eqref{eq:tsstep} follows since $\ts_{S(\ts)^{\mr{c}}} = 0$ by assumption, and \eqref{eq:thresstep} follows from the thresholding step in \eqref{eq:threshlas}.
Therefore, on the event $\mc{G}$, it holds that
\begin{align}
    | S(\tthres) \setminus S(\ts) | \leq \frac{\norm{\tinit - \ts}_1}{\lambda_{\mr{thres}}} \leq \frac{4 \lambda_{\mr{init}} s}{\lambda_{\mr{thres}} \,\phi^2(\mM, s)}. \label{eq:setsize}
\end{align}
\end{proof}

For all $j \in S(\ts)$, on the event $\mc{G}$, we have
\begin{align}
    | (\tinit)_j | &\geq \theta_{\min} - \norm{(\ts - \tinit)_{S(\ts)}}_{\infty} \\
    &\geq \theta_{\min} - \norm{(\ts - \tinit)_{S(\ts)}}_{1} \\
    &\geq \theta_{\min} - \norm{ \ts - \tinit }_{1} \\
    &\geq \theta_{\min} - \frac{4  \lambda_{\mr{init}} s}{\phi^2(\mM, s)}. 
\end{align}

Therefore, if 
\begin{align}
    \lambda_{\mr{thres}} \geq \theta_{\min} - \frac{4  \lambda_{\mr{init}} s}{\phi^2(\mM, s)}, \label{eq:capture}
\end{align}
$S(\tthres) \supseteq S(\ts)$ is satisfied on $\mc{G}$. Combining \lemref{thm:lassol1}, \eqref{eq:setsize}, and \eqref{eq:capture} completes the proof of \thmref{thm:thres}.

\section{Proof of Theorem \ref{thm:ODLinBAI}} \label{app:proofOD}
The proof of \thmref{thm:ODLinBAI} closely follows the proof of \citet[Th.~2]{yang}. Therefore, we only explain the differences, which are as follows.
\begin{enumerate}[label=(\roman*), leftmargin=*]
    \item Due to our construction, $m$ in \citet{yang} is replaced by $\tilde{T} = \left \lfloor \frac{T}{\left \lceil \log_2 d \right \rceil} \right \rceil$. 
    \item Let $\mc{A}_r$ be the active arms in round $r$ and let $\{\va_r(i) \colon i \in \mc{A}_r\} \subset \mathbb{R}^{d_r}$ be the dimensionality-reduced arm vectors. From \citet[Appendix B]{fiez}, it holds that
    \begin{align}
        \max_{i \in \mc{A}_r} \norm{\va_r(i)}_{\mM(\tilde{\pi}^{(r)})^{-1}}^2 \leq d_r \left(1 + \frac{d_r^2}{\tilde{T}}\right),
    \end{align}
    where $\tilde{\pi}^{(r)}$ is the rounded version of the G-optimal design output $\pi^{(r)}$.
    \item In the proof of \citet[Lemma~3]{yang}, the set $\mc{B}_r$ is the set of arms in $\mc{A}_{r-1}$ excluding the best arm and $\lceil \frac{d}{2^{r+1}} \rceil - 1$ suboptimal arms with the largest mean rewards. We re-define $\mc{B}_r$ as the set of arms in $\mc{A}_{r-1}$ excluding the best arm and $\lceil \frac{d}{2^{r}} \rceil - 1$ suboptimal arms with the largest mean rewards.
\end{enumerate}

With the modifications in items (i) and (ii) and following the steps in the proof of \citet[Lemma~2]{yang}, we get for any arm $i \in \mc{A}_{r-1}$ 
\begin{align}
    \Prob{\hat{\mu}_r(1) < \hat{\mu}_r(i) | 1 \in \mc{A}_{r-1}} \leq \exp \left\{ -\frac{\tilde{T} \Delta_i^2}{8 \lceil \frac{d}{2^{r-1}} \rceil (1 + \frac{d^2}{\tilde{T}})} \right \}, \label{eq:probmu}
\end{align}
where $\hat{\mu}_r(i)$ denotes the estimated mean of arm $i$ in round $r$. 

Using item (iii) and \eqref{eq:probmu}, we go through the proof of \citet[Lemma~3]{yang} and get
\begin{align}
    \Prob{1 \notin \mc{A}_r | 1 \in \mc{A}_{r-1}} \leq \begin{cases}
        (K - \frac{d}{2}) \exp \left \{ - \frac{ \tilde{T} \Delta_{\lceil \frac{d}{2^r} \rceil + 1}} {16 ( \lceil \frac{d}{2^r} \rceil + 1) (1 + \frac{d^2}{\tilde{T}})} \right \}, &\text{if } r = 1 \\
        (\frac{d}{2^r} + 1) \exp \left \{ - \frac{ \tilde{T} \Delta_{\lceil \frac{d}{2^r} \rceil + 1}} {16 ( \lceil \frac{d}{2^r} \rceil + 1) (1 + \frac{d^2}{\tilde{T}})} \right \}, &\text{if } r > 1. \label{eq:1notin}
    \end{cases}
\end{align}

Finally, following the steps in the proof of \citet[Th.~2]{yang} with \eqref{eq:1notin}, we get
\begin{align}
    \Prob{\hat{I} \neq 1} &\leq \left(K - \frac{d}{2} + \sum_{r = 2}^{\lceil \log_2 d \rceil} \frac{d}{2^{r}} + \lceil \log_2 d \rceil - 1 \right) \exp \left\{ - \frac{\tilde{T}}{16 \left( 1 + \frac{d^2}{\tilde{T}}\right)  H_{2, \mathrm{lin}}(d)} \right\} \\ 
    &\leq (K + \log_2 d) \exp \left\{ - \frac{\tilde{T}}{16 \left( 1 + \frac{d^2}{\tilde{T}}\right)  H_{2, \mathrm{lin}}(d)} \right\}, 
\end{align}
which completes the proof.

\section{Proof of \corref{cor:main}} \label{app:maincor}
We set $\kappa$, $\lambda_{\mr{init}}$ and $\lambda_{\mr{thres}}$ as
\begin{align}
\kappa &= \frac{b^2}{\theta_{\min}^2} \frac{25}{24}, \label{eq:kappa} \\
\lambda_{\mr{init}} &= \frac{1}{\sqrt{\kappa (s + s^2)}} ,\quad\mbox{and}\\
\lambda_{\mr{thres}} &= \frac{b}{s} \lambda_{\mr{init}}.\label{eq:s1choice}
\end{align}
Note that \eqref{eq:s1choice} sets $s_1 = s + s^2$ in \eqref{eq:s1}, and for any $s \in \mathbb{N}$, we check that the condition in \thmref{thm:main} holds:
\begin{align}
    \theta_{\min} = \frac{b}{\sqrt{\kappa}} \sqrt{\frac{25}{24}} \geq \frac{b}{\sqrt{\kappa}} \frac{s + \frac{1}{s}}{\sqrt{s + s^2}} = \lambda_{\mr{init}}(c + b s).
\end{align}

Next, we would like to set $c_0 = \frac{T_1}{T_2}$ so that the two exponents in \eqref{eq:maineq} are equal. However, since $H_{2, \mr{lin}}(s + s^2)$ is not available to us, we use the lower bound
\begin{align}
    H_{2, \mr{lin}}(s + s^2) = \max_{i \in \{2, \dots, s + s^2\}} \frac{i}{\Delta_i^2} \geq \frac{s + s^2}{\Delta_{s + s^2}^2} \geq  \frac{s + s^2}{4},  \label{eq:H2lower}
\end{align}
where the last inequality follows from the assumption $|\mu_k| \leq 1$ for all $k \in [K]$.\footnote{This step is the only place where the assumption on the mean rewards is used.} One can further upper bound $\Delta_{s + s^2}$ by using the values of $K$ arm vectors and searching for $\ts$ that gives the largest $\Delta_{s + s^2}$.

We set the ratio $c_0 = \frac{T_1}{T_2}$ as
\begin{align}
    c_0 = \frac{25 b^2 x_{\max}^2}{3 \theta_{\min}^2 \log_2(s + s^2)}, \label{eq:c0app}
\end{align}
which together with \eqref{eq:kappa}--\eqref{eq:H2lower} ensures that 
\begin{align}
     \exp\left\{ - \frac{ \left \lfloor \frac{T_2}{\log_2(s_1)} \right \rfloor}{16 \left(1 + \epsilon \right) H_{2, \mr{lin}} (s_1) } \right\} \geq \exp \left\{ - \frac{T_1 \lambda_{\mr{init}}^2}{32 x_{\max}^2}  \right\}. \label{eq:expineq}
\end{align}
Combining \eqref{eq:expineq} with $s_1 = s + s^2$ completes the proof.

The ratio $c_0$ decreasing with $s$ as in \eqref{eq:c0app} is consistent since as $s$ approaches $d$, the sparse linear bandit approaches the standard linear bandit, and we would expect to spend more budget on phase 2 than phase 1 for large $s$. We deliberately choose the parameters in \eqref{eq:kappa}--\eqref{eq:s1choice} to maintain this property.

\section{Implementation Details} \label{app:CV}
In all computations of the Lasso problem \eqref{eq:lassoinit}, we use the ADMM algorithm \citep{boyd2011}.
\subsection{Lasso-OD with $K$-fold Cross-Validation}
In the implementation of Lasso-OD-CV, the ratio of the budgets, $\frac{T_1}{T_2}$, is set to the default value $\frac{1}{4}$, i.e., naturally, the algorithm spends more budget for the BAI algorithm than for the support estimation. 

For tuning the hyperparameters $\lambda_{\mr{init}}$ and $\lambda_{\mr{thres}}$, we pull $T_1$ arms according to the allocation given in \eqref{eq:vtilde}. We use the following cross-validation steps to tune the parameters. 
\begin{enumerate}[label=(\roman*), leftmargin=*]
\item Fix two sets of  hyperparameters  $\{\lambda_{\mr{init}, 1}, \lambda_{\mr{init}, 2}, \dots, \lambda_{\mr{init}, m}\}$ and $\{\lambda_{\mr{thres}, 1}, \lambda_{\mr{thres}, 2}, \dots, \lambda_{\mr{thres}, m}\}$ that are candidates for $\lambda_{\mr{init}}$ and $\lambda_{\mr{thres}}$ respectively. 
\item Iteratively tune the parameters by fixing one of them and searching for the best parameter for the other one.
\item In each cross-validation round, the objective is to minimize the loss function
\begin{align}
    L = \frac{1}{T_1} \mathbb{E}\left[\norm{\rvy - \mX \tthres}_2^2\right] + c_1 \Prob{\norm{\tthres}_0 < s} + c_2 \mathbb{E}\left[1\left\{\norm{\tthres}_0 > s \right\} \norm{\tthres}_0\right], \label{eq:loss}
\end{align}
where $c_1$ and $c_2$ are $\ell_0$-norm regularization parameters. Here, the first term in \eqref{eq:loss} is the mean-squared error; the second and the third terms penalize the $\ell_0$-norm error and force the hyperparameters to output an estimate with $s$ variables. In application, we set $c_1 = 200$ and $c_2 = 5$, giving more importance to detecting at least $s$ variables. If the value of $s$ is not available, we can still use this technique by setting $c_1 = c_2 = 0$. As standard, we approximate \eqref{eq:loss} by training the parameters in $K-1$ blocks and testing in the remaining block. 
\item To fasten the convergence, in each round of cross-validation, we exponentially narrow down the candidate sets. 
\item To reduce the variance in cross-validation, we employ Monte-Carlo simulations, i.e., we independently divide the data into $K$ blocks for multiple times and then take the average loss. 
\end{enumerate}

\subsection{Lasso-OD-Analytical}
In the implementation of Lasso-OD-Analytical, we set the hyperparameters $\lambda_{\mr{init}}$ and $\lambda_{\mr{thres}}$ as in \eqref{eq:kappa}--\eqref{eq:s1choice}, and $c_0$ is set so that \eqref{eq:c0app} holds with equality. This setup effectively uses the hardness parameter $H_{2, \mr{lin}}(s + s^2)$. To compute these hyperparameters, we first need to compute $\phi^2(\mM, s)$. As discussed in Appendix~\ref{app:design}, this requires solving a MIDCP. We do this by using the YALMIP toolbox \citep{yalmip} because it does not ask to convert the problem into another one. Alternatively, one can use the CVX toolbox \citep{boyd} with a little bit more effort, or use \lemref{lem:sigma} and solve an easier problem at the expense of some performance loss.    

\subsection{OD-LinBAI and Other BAI Algorithms}
We implement OD-LinBAI and other BAI algorithms shown in \secref{sec:experiments} using the methods described in \citet[Appendix~E]{yang}.

\paragraph{BayesGap-Adaptive:} In general, BayesGap algorithm \citep{hoffman} requires the knowledge of the hardness parameter. As in \citet{hoffman, yang}, at the beginning of each time instant, we input the estimated hardness parameter according to the three-sigma rule. In the experiments, we omit the oracle version of BayesGap that directly uses the knowledge of the hardness parameter. 
\section{Additional Experiments}\label{app:experiments}
\subsection{Variants of Our Algorithm}
We present two variants of Lasso-OD that modify the operations in phase 2 and one variant that replaces the thresholded Lasso in phase 1. 

\paragraph{Lasso-$\mathcal{XY}$-Allocation:} This algorithm is identical to Lasso-OD except that the $G$-optimal design used to determine the allocations within each round is replaced by the $\mathcal{X}\mathcal{Y}$-allocation from \cite{soare}. Let $\mathcal{X} = \{a(i) \colon i \in \mathcal{A}\}$ be the set of arms in an active set $\mc{A}$. Let $\mc{Y} = \{\vx - \vx' \colon \vx, \vx' \in \mathcal{X}, \vx \neq \vx' \}$ be the set of arm differences. The $\mc{X}\mc{Y}$-allocation solves the problem
\begin{align}
    \pi^* = \argmin_{\pi \in \mc{P}(\mc{A})} \max_{\vy \in \mc{Y}} \norm{\vy}_{\ms{M(\pi)}^{-1}}. \label{eq:XY}
\end{align}
Lasso-$\mc{X} \mc{Y}$-allocation replaces Line 11 of Algorithm~\ref{alg:ODLINBAI} with~\eqref{eq:XY}. In the experiments, we compute~\eqref{eq:XY} using the Frank--Wolfe algorithm as in \cite{fiez}. 

From \citet[Proof of Lemma~2]{yang}, the probability that a sub-optimal arm $i$ has a smaller estimated mean than the optimal arm 1 is bounded as
\begin{align}
    \Prob{\hat{\mu}(1) < \hat{\mu}(i)} &\leq \exp \left \{ - \frac{\Delta_i^2}{2 \norm{\va(1) - \va(i)}_{\mM(\pi)^{-1}}} \right\} \\
    & \leq \exp \left \{ - \frac{\Delta_i^2}{2 \max_{i\neq j} \norm{\va(j) - \va(i)}_{\mM(\pi)^{-1}}} \right\} \label{eq:XYopt} \\
    & \leq \exp \left \{ - \frac{\Delta_i^2}{8 \max_{i \in \mathcal{A}} \norm{\va(i)}_{\mM(\pi)^{-1}}} \right\}, \label{eq:Goptimalopt}
\end{align}
where $\pi$ is the allocation within the round. Here, \eqref{eq:Goptimalopt} follows from the triangle inequality. The G-optimal design optimizes the allocation $\pi$ in \eqref{eq:Goptimalopt}, and $\mc{X}\mc{Y}$-allocation optimizes \eqref{eq:XYopt}. Since $\mc{X}\mc{Y}$-allocation optimizes a tighter bound, Lasso-$\mc{X}\mc{Y}$-allocation is expected to perform better than Lasso-OD. 

\paragraph{Lasso-BayesGap:} Since BayesGap performs better than OD-LinBAI in the examples in \secref{sec:experiments}, we propose the variant Lasso-BayesGap where in phase 2, OD-LinBAI is replaced by BayesGap-Adaptive from \cite{hoffman}. 

In our implementations of Lasso-$\mathcal{XY}$-Allocation and  Lasso-BayesGap (as described above), the parameters of Lasso are tuned via  cross-validation. 

\paragraph{PopArt-OD:} Recently, \cite{jang2022} develop the PopArt algorithm, which estimates the unknown parameter $\ts$ similarly to Lasso and TL. Due to its superior $\ell_1$ error to Lasso, PopArt also guarantees that the  support of $\ts$ is estimated efficiently. Similar to TL, the PopArt estimate is obtained by thresholding an initial estimate. Our variant, PopArt-OD, replaces the TL in phase~1 of Lasso-OD with PopArt, and retains phase~2 as is. We give its pseudo-code in  Algorithm~\ref{alg:popart} and analyze its performance in the section below. Line 1 of PopArt involves an optimization problem that yields the optimal covariance matrix with respect to an upper bound on the error probability; this process reflects the design matrix optimization of TL described in Appendix~\ref{app:design}. In PopArt, if the population covariance $\mM$ in Line 3 was replaced with the empirical covariance and if the Catoni estimator in Line 4 was replaced with averaging, the resulting algorithm would be the thresholded OLS estimator. 

PopArt has one hyperparameter (the estimate on the error probability  $\delta$ in \citet{jang2022}). In the experiments below, we tune the hyperparameter using a $K$-fold cross validation procedure. 

\begin{algorithm}[!htbp] 
   \caption{PopArt-OD}
   \label{alg:popart}
\begin{algorithmic}[1]
\INPUT Time budget $T$, arm vectors $\va(1), \dots, \va(K) \in \mathbb{R}^d$, $\theta_{\min}$, and sparsity $s$.
\STATE Solve the convex optimization problem $\nu^* = \argmin \limits_{\nu \in \mc{P}[K]} \max \limits_{i \in [d]} \left( \{\sum_{i = 1}^K \nu_i \va(i) \va(i)^\top\}^{-1} \right)_{ii}$. Let the objective value of the minimum be $H_2^*$ and $\mM = \sum_{i = 1}^K \nu_i^* \va(i) \va(i)^\top$. Set
\begin{align}
    \lambda_{\mr{PA}} &= \min \left\{ \sqrt{2 H^2_*}, \frac{\theta_{\min}}{2} \right\}, \quad 
    c_{\mr{PA}} = \frac{2 H^2_*} {\lambda_{\mathrm{PA}}^2 s \log_2 s}, \quad
    T_1 = \left\lceil T \frac{c_\mr{PA}}{1 + c_{\mr{PA}}} \right\rceil, \quad T_2 = T - T_1, \quad g = \frac{\lambda_{\mathrm{PA}}^2}{8 H^2_*}. \label{eq:cpop}
\end{align}
\STATE Sample $T_1$ arms, $A_1, \dots, A_{T_1}$ i.i.d. with $\nu^*$, and observe rewards $\ry_1, \dots, \ry_{T_1}$.
\STATE
For $t = 1, \dots, T_1$, let
$\tilde{\bm{\theta}}_t = \mM^{-1} \va(A_t) \ry_t \in \mathbb{R}^d$.
\STATE Set for $i \in [d]$, $\theta'_i = \mathrm{Catoni}\left((\tilde{\bm{\theta}}_{1, i}, \dots, \tilde{\bm{\theta}}_{T_1, i}), \sqrt{\frac{g  }{(\mM^{-1})_{ii} \left(1 + \frac{2 g}{1 - 2 g} \right) }}\right)$, where 
$\mathrm{Catoni}((Z_1, \dots, Z_n), \alpha)$ is the unique solution $y$ to the equation
\begin{align}
    \sum_{i = 1}^n \psi(\alpha(Z_i - y)) = 0, \quad 
    \psi(x) = \mathrm{sign}(x)(1 + |x| + x^2/2).
\end{align}
\STATE $\hat{\bm{\theta}}_{\mr{PA}} = \Big( \theta'_i 1\big\{ |\theta'_i| \geq \sqrt{8 \mM^{-1}_{ii} g} \big\} \colon i \in [d]   \Big)$ and $\hat{\mathcal{S}}_{\mr{PA}} = S(\hat{\bm{\theta}}_{\mr{PA}})$. 
\STATE Run OD-LinBAI (Algorithm~\ref{alg:ODLINBAI}) restricted to the support $\hat{\mathcal{S}}_{\mr{PA}}$ using $T_2$ pulls.
\OUTPUT $\hat{I}$ is the only remaining arm as the output of Algorithm~\ref{alg:ODLINBAI}. 
\end{algorithmic}
\end{algorithm}

\subsection{Analysis of PopArt-OD}
The following theorem bounds the error probability of PopArt-OD.
\begin{theorem}\label{thm:popart}
    For any linear bandit instance, the error probability of PopArt-OD given in Algorithm~\ref{alg:popart} is bounded as
    \begin{align}
        \Prob{\hat{I} \neq 1} \leq (K + \log_2 d + 2d) \exp \left \{ - \frac{T}{16 \lfloor \log_2(s) \rfloor (1 + \epsilon_{\mr{POP}}) H_{2, \mathrm{lin}}(s) (1 +  c_{\mr{PA}})} \right \}, 
    \end{align}
    where $c_{\mr{PA}}$ is defined in \eqref{eq:cpop}, below, and $\epsilon = \frac{(1 + c_{\mr{POP}}) s^2}{T}$.
   
\end{theorem}

\begin{proof}
    By Line 5 of PopArt, if $i \in [d]$ satisfies that $\theta_i' \geq \lambda_{\mr{PA}}$, then $i \in \hat{S}_{\mr{PA}}$. From \citet[Prop.~1 and Th.~1]{jang2022}, with probability at least $1 - 2d \exp \left\{ - \frac{T_1 \lambda_{\mr{PA}}^2}{8 H_*^2} \right \} $,
    the initial PopArt estimator $\bm{\theta}'$ satisfies
    \begin{align} 
    \norm{\bm{\theta}' - \ts}_{\infty} \leq \lambda_{\mr{PA}}.
    \end{align}
    and the PopArt estimator satisfies $S(\hat{\bm{\theta}}_{\mr{PA}}) \subseteq S(\ts)$. By selecting $\lambda_{\mr{PA}} \leq \frac{\theta_{\min}}{2}$, we further ensure that $S(\ts) \subseteq S(\hat{\bm{\theta}}_{\mr{PA}})$, giving $S(\hat{\bm{\theta}}_{\mr{PA}}) = S(\ts)$. Therefore, by the union bound and \thmref{thm:ODLinBAI}, the error probability of PopArt-OD is bounded as
    \begin{align}
        \Prob{\hat{I} \neq 1} &\leq \Prob{S(\hat{\bm{\theta}}_{\mr{PA}}) \neq S(\ts)} + \Prob{\hat{I} \neq 1 \middle| S(\hat{\bm{\theta}}_{\mr{PA}}) = S(\ts)} \\
        &\leq 2d \exp \left\{ - \frac{T_1 \lambda_{\mr{PA}}^2}{8 H_*^2} \right \} + (K + \log_2 d) \exp \left\{ - \frac{T_2}{16 \lceil \log_2 s \rceil H_{2, \mr{lin}}(s) (1 + \epsilon_{\mr{PA}})}\right \}. \label{eq:popeq}
    \end{align}
    The rest of the proof follows steps similar to the proof of \thmref{thm:main}, which aims to balance the two exponents in \eqref{eq:popeq}. 
    The choices of $c_{\mr{PA}}, T_1$, and $T_2$ together with the lower bound $H_{2, \mr{lin}}(s) \geq \frac{s}{4}$ (see, \eqref{eq:H2lower}) imply that 
    \begin{align}
        \frac{T_1 \lambda_{\mr{PA}}^2}{8 H_*^2} \geq \frac{T_2}{16 \lceil \log_2 s \rceil H_{2, \mr{lin}}(s) (1 + \epsilon_{\mr{PA}})},
    \end{align}
    which completes the proof. Note that we select $\lambda_{\mr{PA}} \leq \sqrt{2 H_*^2}$ to ensure that $1 - 2g \geq \frac{1}{2} > 0$, making the Catoni parameter in Line~4 of Algorithm~\ref{alg:popart} valid. 
\end{proof}

Theorem~\ref{thm:popart} shows that PopArt-OD has an error probability that scales as
\begin{align}
    \exp\left\{ -\Omega\left( \frac{T}{(\log_2 s) H_{2, \mr{lin}}(s)}\right) \right \}
\end{align}
as $T$ and $s$ grow. This is the same theoretical result as that for Lasso-OD. However, we observe from the next section (specifically Appendix~\ref{sec:first_ex}) that Lasso-OD outperforms PopArt-OD empirically.

\subsection{Experiments}
In the experiments below, we include Lasso-$\mc{X}\mc{Y}$-allocation and Lasso-BayesGap to the list of algorithms in Section~\ref{sec:experiments}.

\subsubsection{First Example} \label{sec:first_ex}
In the first example, we test the performance of the various BAI algorithms for sparsities of at least 2. We generate $K$ $d$-dimensional arm vectors $a(k) = (a(k)_i\colon i \in [d])$, $k \in [K]$, where $a(k)_i$'s are distributed $\mc{N}(0, \frac{1}{s})$ independent across arms $k \in [K]$ and coordinates $i \in [d]$. The $s$-sparse unknown vector $\ts$ is set as $(\theta^*)_i = \frac{1}{\sqrt{s}}$ for $i \in [s]$ and $(\theta^*)_i = 0$ for $i = s+1, \dots, d$. Figure~\ref{fig:plots_s} compares the performances of several algorithms in the literature and variants of our algorithm for $s \in \{2, 3, 4\}$, $K = 50$, $d \in \{10, 20\}$, and $T \in [200, 10^4]$. 
\begin{figure}[t]
\includegraphics[width=1\textwidth]{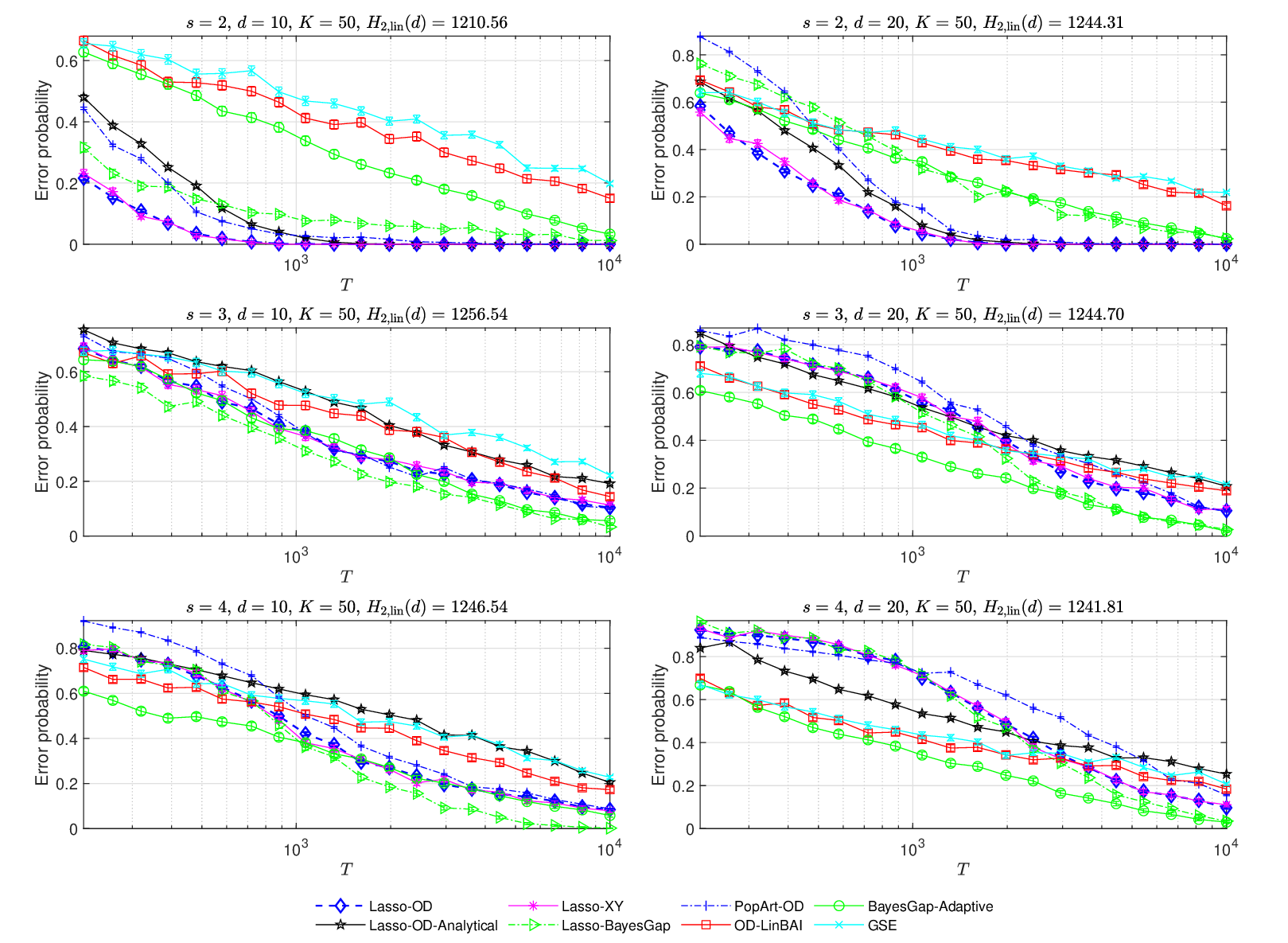}
\centering
\caption{Comparison of several algorithms with $s \in \{2, 3, 4\}$.} \vspace{-.1in}
\label{fig:plots_s}
\end{figure}
For all bandit instances under consideration, the performances of Lasso-OD and Lasso-$\mc{X}\mc{Y}$-allocation are almost identical. However, due to its low computational complexity (see tables below for the CPU runtimes), Lasso-OD is preferred over Lasso-$\mc{X}\mc{Y}$-allocation. 
Among different variants of Lasso-OD, Lasso-OD and Lasso-$\mc{X}\mc{Y}$-allocation have the best performance for the instances  with $s = 2$. For $s = 3$ and $s = 4$, among the algorithms shown, Lasso-BayesGap performs the best for a large enough budget $T$. The poor performance of Lasso-based algorithms for small budgets is because at larger $s$, the minimum budget that should be allocated to phase 1 to reliably estimate the support of $\ts$ increases with $s$. For the instances with $s \in \{3, 4\}$, BayesGap-Adaptive performs remarkably well, but it is outperformed by Lasso-based algorithms for $s = 2$. PopArt-OD algorithm is outperformed by Lasso-OD for all instances shown, which implies that the variable selection property of PopArt is poorer than that of the thresholded Lasso. 

In Tables~\ref{tab:CPU1} and~\ref{tab:CPU2},\footnote{Pre-calculation in  Tables~\ref{tab:CPU1} and~\ref{tab:CPU2}  refers to the calculation of $\phi^2(\mM, s)$ that is used to determine the hyperparameters for Lasso-OD-Analytical.} we report the average CPU runtimes for the instances in the first example with $s = 2$ and $s = 3$. Lasso-OD is superior to all other algorithms in terms of the computational complexity.\footnote{All experiments are implemented on MATLAB 2023a on an Intel(R) Core(TM) i9-12900H processor.}

\subsubsection{Second Example}
In the second example, we assume that $\ts$ belongs to the finite set $\mc{H} \triangleq \{\theta \in \mathbb{R}^d \colon \norm{\theta}_0 = s, \theta_i \in \{-\frac{1}{\sqrt{s}}, 0, \frac{1}{\sqrt{s}} \text{ for } i \in [d] \}\}$. In other words, the non-zero coordinates of $\ts$ are assumed to have magnitude $\frac{1}{\sqrt{s}}$. We generate $K$ $d$-dimensional arm vectors in the vicinity of $\mc{H}$ as follows: $a(k)_i = R_{k, i} \cos(\pi/4 + Z_{k, i})$, where $R_{k, i}$'s are independently and identically distributed (i.i.d.) generated with distribution $\mathrm{Unif}(\{-1, 1\})$, $Z_{k, i}$'s are i.i.d.\ generated with distribution $\mc{N}(0, 0.01)$, and $R_{k, i}$'s and $Z_{k, i}$'s are independent. For this bandit instance, given the arm vectors and using the assumption that $\ts \in \mc{H}$, we can lower bound the hardness parameter $H_{2, \mr{lin}}(s + s^2)$ by computing the minimum hardness parameter for the vectors $\theta \in \mc{H}$. In Figure~\ref{fig:plots_LB}, Lasso-OD-An.-LB computes the hyperparameters analytically and obtains $T_1$ from the lower bound on $H_{2, \mr{lin}}(s + s^2)$ above instead of the true value of $H_{2, \mr{lin}}(s + s^2)$. Figure~\ref{fig:plots_LB} shows that Lasso-OD-An.-LB outperforms all other algorithms in the literature and achieves similar performance as Lasso-OD and Lasso-$\mc{XY}$-allocation for a large enough time budget. 
\begin{figure}[t]
\includegraphics[width=1\textwidth]{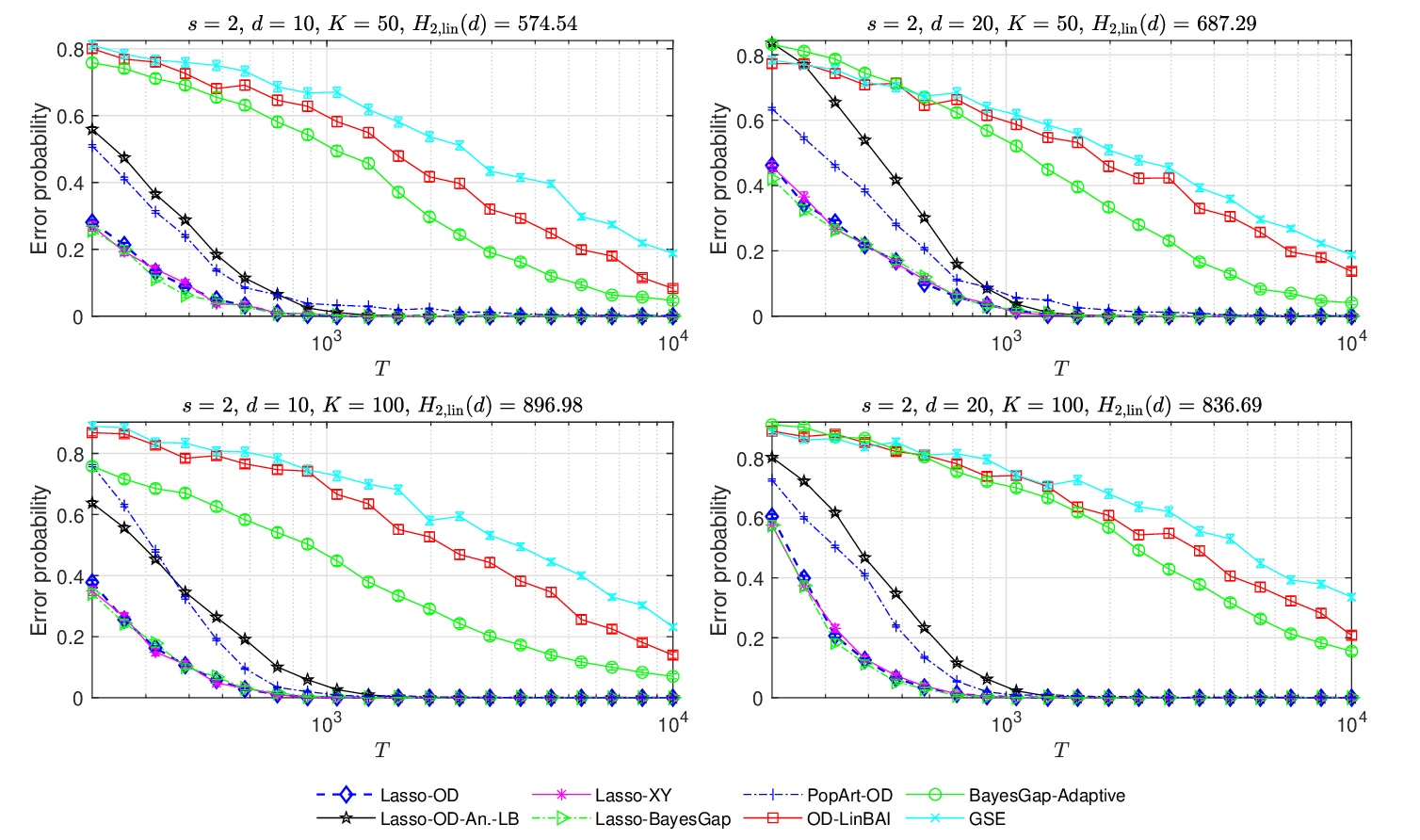}
\centering
\caption{Comparison of several algorithms with $\ts$ belonging to a finite set.} 
\label{fig:plots_LB}
\end{figure}

\subsubsection{Third Example}
In the third example, we extend the example in \citet{yang, jedra, fiez} to sparse linear bandits. We set $\ts = (\frac{1}{\sqrt{2}}, \frac{1}{\sqrt{2}}, 0, \dots, 0)$, i.e., $s = 2$, and $\mc{S} = S(\ts) = \{1, 2\}$. For the coordinates in $\mc{S}$, we pull arms as in \cite{yang}; we set $a(1)_{\mc{S}} = (\cos(\pi /4), \sin(\pi/4))$, $a(K)_{\mc{S}} = (\cos(5 \pi / 4), \sin(5 \pi / 4))$, and $a(i)_{\mc{S}} = (\cos(\pi/ 2 + \phi_i), \sin(\pi/2 + \phi_i))$ for $i = 2, \dots, K-1$, where $\phi_i$ are independently drawn from $\mc{N}(0, 0.09)$. For any $i \in [K]$, we draw $a(i)_{\mc{S}^{\mr{c}}}$ independently from the uniform distribution on the $(d-s)$-dimensional centered sphere of radius $\sqrt{\frac{d-s}{s}}$. Recall that since $\ts_{\mc{S}^{\mr{c}}} = 0$, the values of arms on the coordinates $\mc{S}^{\mr{c}}$ have no effect on the best arm or the value of the hardness parameter. The problem would be identical to that in \citet{yang} if the agent knew the support $\mc{S}$. In this bandit instance, arm 1 is the best arm and there are $K-2$ arms whose mean values are close to that of the second best arm. In the non-sparse case, i.e., $d = s = 2$, \citet{yang} demonstrate that OD-LinBAI outperforms the other algorithms. Figure~\ref{fig:plots_cos} compares the performance of variants of our algorithm with the other algorithms in the literature. We report the empirical performances for $d \in \{10, 20\}, K \in \{50, 100\}$, and $T \in [200, 10^4]$. Among the algorithms shown, Lasso-OD and its variants Lasso-$\mc{XY}$-allocation and PopArt-OD significantly outperform the other algorithms. Unlike the previous two examples, for this example, Lasso-BayesGap is not the best performing algorithm. 

\begin{figure}[!htbp]
\includegraphics[width=1\textwidth]{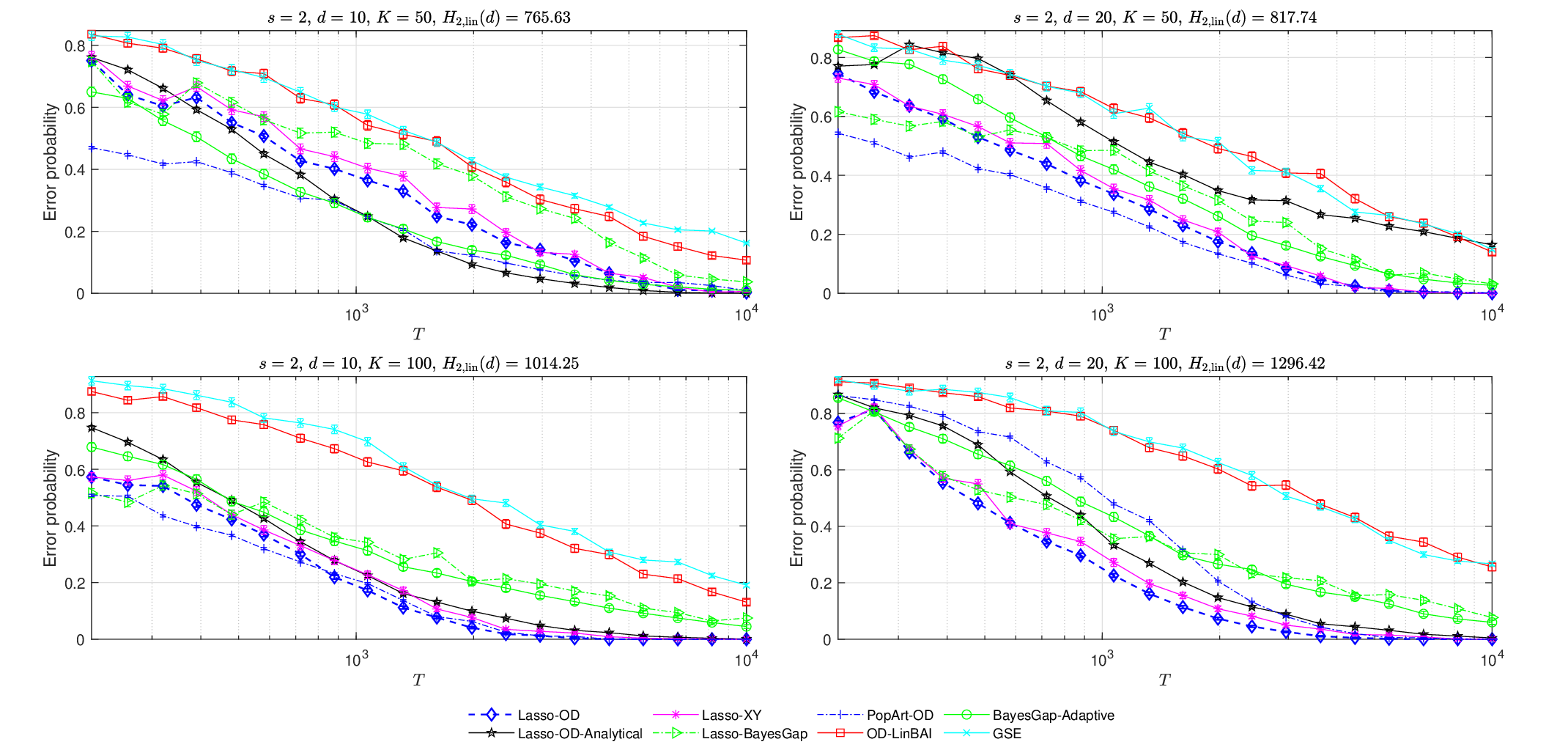}
\centering
\caption{Comparison of several algorithms for the example bandit instance in \cite{yang}.} 
\label{fig:plots_cos}
\end{figure}

\begin{table}
  \caption{The empirical means of the CPU runtimes for $s = 2$, $d = 10$, $K = 50$.}
  \label{tab:CPU1}
  \vskip 0.15in
  \centering
  \begin{tabular}{rrrrrrrrrr} 
    \toprule
    &  \multicolumn{9}{c}{CPU runtimes (milliseconds) }                   \\
    \cmidrule(r){2-10}
    $T$ & \scriptsize Pre-calc. & \scriptsize
Lasso Tuning & \scriptsize
Lasso-OD & \scriptsize
Lasso-$\mc{XY}$ &\scriptsize
Lasso-BayesG. & \scriptsize
Lasso-OD-An. & \scriptsize
BayesGap-Ad. & \scriptsize OD-LinBAI &\scriptsize GSE
\\
\midrule
100  & 9680            & 3300         & 0.98     & 9.0                   & 1.6         & 1.7       & 3.2      & 2.2       & 4.0    \\
200  & 9680            & 3500         & 0.61     & 5.2                 & 2.9         & 1.2       & 5.1      & 2.2       & 4.0    \\
400  & 9680            & 3800         & 0.55     & 3.6                 & 5.5         & 0.83      & 9        & 2.1       & 4.0    \\
800  & 9680            & 4380         & 0.48     & 3.4                 & 11          & 0.76      & 17       & 2.1       & 4.0    \\
1600 & 9680            & 6570         & 0.66     & 4.1                 & 23          & 0.68      & 34       & 2.1       & 4.0    \\
3200 & 9680            & 7520         & 0.66     & 4.0                   & 46          & 0.70       & 69       & 2.1       & 4.0    \\
6400 & 9680            & 7710         & 0.70      & 4.1                 & 89          & 1.1       & 139      & 2.1       & 4.0    \\
\bottomrule
  \end{tabular}
  \vskip -0.1in
\end{table}

\begin{table}
  \caption{The empirical means of the CPU runtimes for $s = 3$, $d = 10$, $K = 50$.}
  \label{tab:CPU2}
  \vskip 0.15in
  \centering
  \begin{tabular}{rrrrrrrrrr} 
    \toprule
    &  \multicolumn{9}{c}{CPU runtimes (milliseconds)}                   \\
    \cmidrule(r){2-10}
    $T$ & \scriptsize Pre-calc. & \scriptsize
Lasso Tuning & \scriptsize
Lasso-OD & \scriptsize
Lasso-$\mc{XY}$ &\scriptsize
Lasso-BayesG. & \scriptsize
Lasso-OD-An. & \scriptsize
BayesGap-Ad. & \scriptsize OD-LinBAI &\scriptsize GSE
\\
\midrule
100  & 27100           & 4200        & 1.1      & 12                  & 1.6         & 2.2       & 2.5      & 2.4       & 4.2  \\
200  & 27100           & 4500         & 0.96     & 8.9                 & 2.9         & 2.1       & 5.2      & 2.4       & 4.1  \\
400  & 27100           & 4900         & 0.94     & 8.1                 & 5.7         & 2.3       & 10       & 2.4       & 4.1  \\
800  & 27100           & 5000         & 0.98     & 7.6                 & 12          & 2.3       & 17       & 2.4       & 4.2  \\
1600 & 27100           & 5100         & 1.5      & 9.2                 & 24          & 2.1       & 34       & 2.4       & 4.2  \\
3200 & 27100           & 7000         & 1.5      & 9.2                 & 47          & 2         & 68       & 2.4       & 4.2  \\
6400 & 27100           & 9800         & 1.6      & 9.4                 & 93          & 2         & 137      & 2.4       & 4.1  \\
\bottomrule
  \end{tabular}
  \vskip -0.1in
\end{table}

\subsubsection{Fourth Example}
In the final example, we test the performance of thresholded Lasso in which the whole horizon of length $T$ is used for learning the support of $\ts$. We draw each entry of the design matrix $\mX \in \mathbb{R}^{T \times d}$ i.i.d.\ from $\mathcal{N}(0, \frac{1}{s})$ and set $\ts = (\frac{1}{\sqrt{s}}, \dots, \frac{1}{\sqrt{s}}, 0, \dots, 0)$ where $\ts$ has $s$ non-zero entries. In  Figure~\ref{fig:lasso}, we report the empirical probability of detection error $\Prob{S(\tthres) \not\supseteq S(\ts)}$ and the empirical mean $\mathbb{E}[|S(\tthres)|]$ over 10,000 independent trials. For $s = 2$, the empirical error probability is 0 for $T \geq 400$; for $s = 4$, the empirical error probability is 0 for $T \geq 800$. Figure~\ref{fig:lasso} shows that for $T \geq 100$ and $s \in \{2, 4\}$, thresholded Lasso is capable of correctly detecting the active variables in $\ts$ with high probability while also keeping the average number of false positives close to zero. As expected, the average number of false positives increases with~$s$. 

\begin{figure}[t]
\includegraphics[width=1\textwidth]{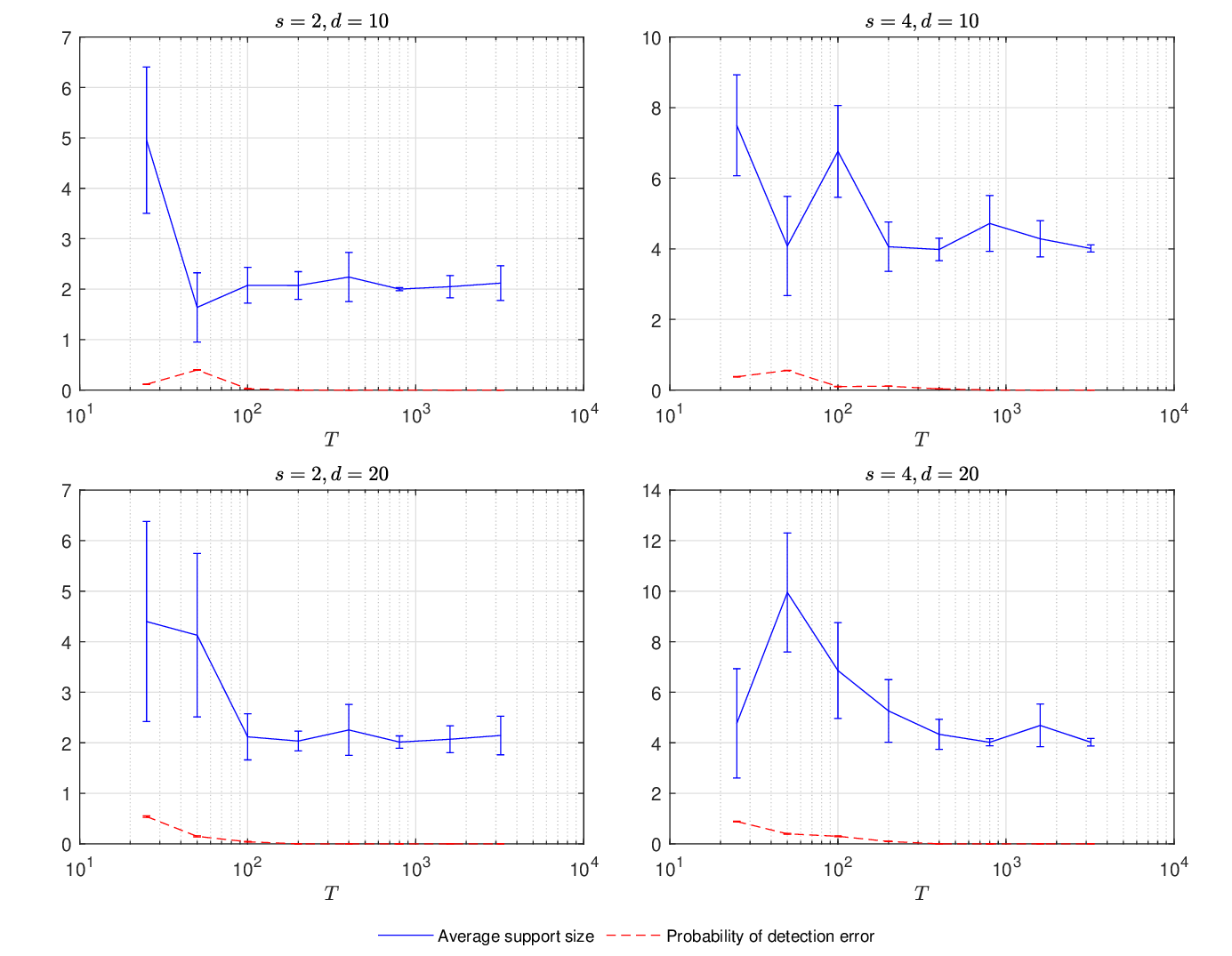}
\centering
\caption{The empirical detection error probability and the empirical size of the thresholded Lasso output.} 
\label{fig:lasso}
\end{figure}

\end{document}

%% file: math_commands.tex
\usepackage{amsmath,amssymb, mathtools, amsfonts,bm, enumitem, amsthm, graphicx, booktabs}
\newtheorem{definition}{Definition}
\newtheorem{lemma}{Lemma}

\newtheorem{corollary}{Corollary}
\newtheorem{theorem}{Theorem}
\newcommand{\thmref}[1]{Theorem~\ref{#1}}

\newcommand{\secref}[1]{Section~\ref{#1}}
\newcommand{\lemref}[1]{Lemma~\ref{#1}}

\newcommand{\corref}[1]{Corollary~\ref{#1}}
\newcommand{\appref}[1]{Appendix~\ref{#1}}
\newcommand{\Prob}[1]{\mathbb P\left[#1\right]}

\newcommand{\tinit}{\hat{\vtheta}_{\mr{init}}}
\newcommand{\tthres}{\hat{\vtheta}_{\mr{thres}}}
\newcommand{\mr}[1]{\mathrm{#1}}
\newcommand{\mc}[1]{\mathcal{#1}}
\newcommand{\ms}[1]{\mathsf{#1}}
\newcommand{\norm}[1]{\left \lVert #1 \right \rVert}
\newcommand{\ts}{\vtheta^*}
\newcommand{\that}{\hat{\vtheta}}

\def\secref#1{section~\ref{#1}}

\def\eqref#1{(\ref{#1})}

\def\1{\bm{1}}

\def\ry{{\textnormal{y}}}

\def\rvy{{\mathbf{y}}}

\def\vtheta{{\bm{\theta}}}
\def\va{{\bm{a}}}

\def\vx{{\bm{x}}}
\def\vy{{\bm{y}}}

\def\mA{{\bm{A}}}

\def\mD{{\bm{D}}}

\def\mM{{\bm{M}}}

\def\mU{{\bm{U}}}
\def\mV{{\bm{V}}}

\def\mX{{\bm{X}}}

\DeclareMathAlphabet{\mathsfit}{\encodingdefault}{\sfdefault}{m}{sl}
\SetMathAlphabet{\mathsfit}{bold}{\encodingdefault}{\sfdefault}{bx}{n}

\DeclareMathOperator*{\argmax}{arg\,max}
\DeclareMathOperator*{\argmin}{arg\,min}